\documentclass[twoside]{article}
\usepackage[preprint]{aistats2027}
\usepackage{amsmath,amssymb,mathtools}
\usepackage[round]{natbib}
\usepackage{booktabs}
\usepackage{graphicx}
\usepackage{subcaption}
\usepackage{microtype}
\usepackage{url}
\usepackage[normalem]{ulem}
\usepackage{xcolor}
\usepackage{hyperref}
\usepackage{cleveref}
\usepackage{enumitem}
\usepackage{tikz}
\usetikzlibrary{decorations.pathreplacing}

\renewcommand{\refname}{References}

\graphicspath{{figures/}}

\definecolor{linkcolor}{RGB}{44,96,163}
\hypersetup{
    colorlinks=true,
    citecolor=linkcolor,
    linkcolor=linkcolor,
    urlcolor=linkcolor,
}

\definecolor{changecolor}{HTML}{A327F5}

\newcommand{\R}{\mathbb{R}}
\newcommand{\E}{\mathbb{E}}

\newcommand{\eg}{\emph{e.g.}}
\newcommand{\ie}{\emph{i.e.}}

\newcommand{\paperauthors}{
  Hamza Cherkaoui\textsuperscript{*} \And
  H\'el\`ene Halconruy\textsuperscript{*} \And
  Antonio Ocello
}

\newcommand{\paperaddresses}{
  SAMOVAR, T\'el\'ecom SudParis\\
  Institut Polytechnique de Paris\\
  Palaiseau, France\\
  \texttt{hamza.cherkaoui@telecom-sudparis.eu}
  \And
  SAMOVAR, T\'el\'ecom SudParis\\
  Institut Polytechnique de Paris\\
  Modal'X, Universit\'e Paris Nanterre\\
  Nanterre, France\\
  \texttt{helene.halconruy@telecom-sudparis.eu}
  \And
  CREST, ENSAE Paris\\
  Institut Polytechnique de Paris\\
  Palaiseau, France\\
  \texttt{antonio.ocello@ensae.fr}
}

\begin{document}

\runningtitle{Radial Tail Repair: Post-Training Correction for Diffusion and Flow Models}
\runningauthor{Cherkaoui, Halconruy, Ocello}

\twocolumn[
\aistatstitle{Radial Tail Repair: Post-Training Correction for Diffusion and Flow Models}
\aistatsauthor{\paperauthors}
\aistatsaddress{\paperaddresses}
]

\begingroup
    \begin{NoHyper}
    \renewcommand{\thefootnote}{*}
    \footnotetext{Equal contribution.}
    \end{NoHyper}
\endgroup

%------------------------------------------------------------------------------------------------------------------
% ABSTRACT
\begin{abstract}
    Diffusion and flow-matching models may struggle to accurately represent distributional tails, especially for heavy-tailed data.
    Existing approaches often use heavy-tailed noise or modify the generative dynamics, requiring dedicated retraining and more difficult estimation.
    We take a simpler approach: rather than relearning the full distribution, we correct only the model's behavior in the extreme regime.
    Our method builds on the radial-angular decomposition of multivariate heavy tails, where radius and direction become asymptotically independent.
    This allows us to preserve the directional structure learned by the pretrained model and correct only its radial tail.
    We estimate the tail onset from reference and generated samples, then rescale generated extremes so that their radial survival matches the target distribution.
    The resulting method is a lightweight post-training correction applied to a frozen diffusion or flow-matching model, requiring neither retraining nor changes to its sampling dynamics.
    Experiments on challenging synthetic heavy-tailed distributions and real weather and wildfire datasets show that this simple correction substantially improves extreme-tail coverage using only a small number of fitting samples.
\end{abstract}

%------------------------------------------------------------------------------------------------------------------
% SECTIONS
\section{Introduction}
\label{sec:introduction}

Extreme events occupy a peculiar place in statistical modeling: they are rare by definition, yet often drive the quantities most relevant to decision making.
Weather and climate extremes such as heatwaves, floods, or severe storms can have major societal and economic consequences \citep[see, \eg,][]{IPCC2022}, while rare market movements are central to risk assessment and stress testing in finance and insurance \citep{embrechts1999modelling}.
More broadly, reliability, environmental sciences, and risk management require accurate estimation of tail probabilities and extreme quantiles \citep{coles2001introduction,mcneil2015quantitative}.
This has motivated a long-standing literature, from classical extreme value theory to modern multivariate methods, focused on tail estimation, extrapolation beyond the observed range, and prediction of rare events \citep{deHaan2006extreme,heffernan2004conditional,davison2015statistics}.
Modern generative modeling extends this problem beyond estimation: the goal is now to learn and sample from the full data distribution, including its extreme regions. 
Diffusion \citep{sohl2015deep,song2019generative,ho2020denoising,song2020score,song2021denoising} and flow-based models \citep{lipman2023flowmatching,boffi2023probability,albergo2025stochastic,liu2023flowstraight} have achieved remarkable performance on complex high-dimensional data, but are primarily designed to capture the overall distribution rather than its extremes.

Accurately generating rare events therefore poses an additional challenge: these models may reproduce the bulk well while still misrepresenting the frequency or magnitude of tail events.
Indeed, performance in low-probability regions can deteriorate substantially, as empirically illustrated by \citet{shariatian2025dlpm}.
This is also reflected in convergence analyses, which typically rely on score regularity or confining assumptions, such as dissipativity or weak log-concavity, rather than directly controlling tail probabilities \citep{silveri2025beyond,arsenyan2026assessing,strasman2026forgetting,stephanovitch2026lipschitz}.
These standard models may consequently struggle to capture heavy-tailed targets, particularly when relying on Gaussian reference distributions.
Recent approaches address this by modifying the reference distribution, perturbation process, training objective, or data representation \citep{yoon2023score,pandey2025heavytailed,shariatian2025dlpm}.
These methods, however, typically require dedicated model design and retraining, even when the bulk is already well captured.

Such heavy-tailed constructions can also be computationally demanding and require substantial data precisely in regions where observations are scarce.
Moreover, \citet{cherkaoui2026heavytails} show that the benefits of heavy-tailed initialization can vanish under finite-sample estimation, as the increased learning difficulty offsets the improved initialization.
This naturally raises the question:

\begin{center}
    \emph{Can we preserve the bulk while efficiently repairing the tail from few samples?}
\end{center}

In this work, we focus on \textit{post-training correction} of the tail behavior of generative models.
Rather than modifying or retraining the whole generative model, we ask whether its tail can be \emph{repaired after training}.
Our approach finds its theoretical foundations in multivariate extreme value theory.
Under multivariate regular variation, extreme observations exhibit a joint radial--angular limit in which the \emph{radius is Pareto-type and asymptotically independent of the angle}.
This fundamental property suggests a simple principle for tail correction: once a sufficiently extreme region has been reached, the radial behavior may be corrected separately from the angular structure.

The method consists of three main steps: \emph{Detect}, \emph{Model}, and \emph{Repair}.
The method identifies a region where the radial and angular components are approximately independent, fits the radial tail behavior, and then applies a one-dimensional tail transformation to the radius while preserving the angular component.
In this way, we correct frequency and magnitude of extreme samples without modifying the dependence geometry learned by the generative model nor its behavior in the bulk.
This procedure being \emph{model-agnostic} can be applied directly to pre-trained diffusion or flow model, without modifying its architecture or retraining its parameters.

\section{Related Work}
\label{sec:related}

Our approach lies at the intersection of heavy-tailed generative modeling and post-training distribution correction.
We briefly summarize these lines of work and highlight their differences from our approach; a more detailed discussion is provided in \autoref{appendix-related-work}.

When dealing with \emph{heavy-tailed generation}, the main line of research incorporates tail behavior directly into the model design and training procedure.
These approaches typically train a model under a modified perturbation process rather than correct an existing generator. \citet{shariatian2025dlpm} replace Gaussian perturbations with $\alpha$-stable noise, while \citet{pandey2025heavytailed} introduce Student-$t$ perturbations and develop the $t$-EDM and $t$-Flow frameworks.
Related developments include the jump-diffusion construction of \citet{baule2025jump} and the refinement of Student-$t$ sampling dynamics through state-dependent diffusion coefficients by \citet{wakatsuki2026selfregulating}.
However, incorporating heavy-tailed noise does not guarantee improved performance.
Through a combined theoretical and empirical comparison, \citet{cherkaoui2026heavytails} identify a trade-off between initialization and statistical estimation: potential benefits of heavy-tailed initialization can be offset by a harder learning problem and less favorable sampling-error bounds.
\citet{fassina2026generalizing} propose a hybrid approach using a normalizing flow to learn a heavy-tailed initialization prior while retaining standard SGM training to refine structural details.

Building on the strong performance of standard diffusion and flow models, a second line of research adapts pretrained generative models to specific tasks or desired distributional properties through \emph{fine-tuning and inference-time guidance}.
These methods build on capabilities acquired during pretraining, either by updating model parameters or by learning an auxiliary correction to the sampling dynamics.
For example, \citet{denker2024deft} learn a compact generalized $h$-transform while retaining the pretrained diffusion model, whereas \citet{guo2026dohf} incorporate Doob-$h$ guidance through iterative distillation.
Other approaches use Feynman--Kac weighted simulation \citep{skreta2025feynman}, conditional reward expectations \citep{ouyang2026alignment}, or regularized estimation of the Doob $h$-function \citep{chang2026doob,suchkov2026weirdo}.

Beyond diffusion-specific adaptation, a complementary line of work revisits the data to identify and correct discrepancies in a trained generative distribution. Discriminator rejection sampling \citep{azadi2019discriminator} and Metropolis--Hastings GANs \citep{turner2019metropolis} use discriminator-based density-ratio estimates to correct generated samples.
For diffusion models, \citet{kim2023discriminator} train a discriminator after pretraining and use its gradient to correct the score during sampling. Similarly, \citet{beltranvelez2026reward} derive a density-ratio reward from a data-versus-model discriminator and use it to fine-tune a flow-matching model.
Our approach shares the objective of using data to correct distributional features that the original training procedure has not adequately captured, but targets tail behavior directly through a radial transformation of generated samples while preserving their angular component and leaving the pretrained generator unchanged.

\section{Methodology}
\label{sec:methodology}

Our \emph{Radial Tail Repair (RTR)} method can be divided into three pillars: \textbf{Detect}, \textbf{Model}, and \textbf{Repair} (see \autoref{fig:pipeline} for a schematic description of our pipeline).
Using two consequences of the radial--angular limit, a central result in the multivariate regular variation (MRV) theory, we select the specific regions that need correction, before acting on them.
First, exploiting the asymptotic independence between radius and angle, we \textbf{Detect} the onset of the MRV regime.
Then, within this selected regime, we \textbf{Model} the radial tail using the Pareto-type approximation of the conditional distribution of the radius beyond a high threshold.
Together, these two steps identify and model the extreme region in which the asymptotic theory is informative and carefully selected the possibly unestimated areas in the pretrained model.
Finally, \textbf{Repair} corrects the generated radial component according to the fitted tail model while preserving the angular component, thereby modifying tail magnitude without altering the learned extremal geometry.

\begin{figure}[t]
    \centering
    \includegraphics[scale=0.43]{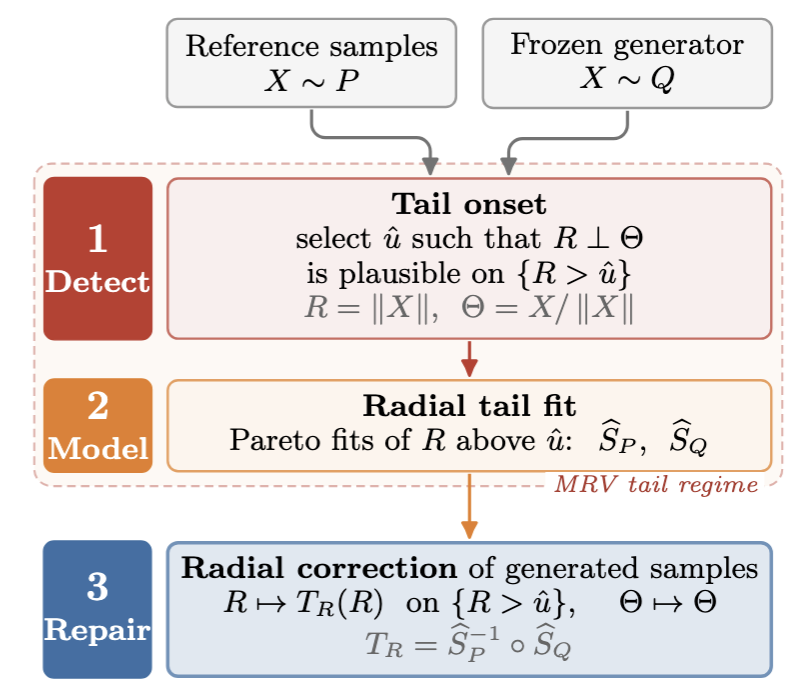}
    \caption{Radial Tail Repair pipeline. A tail onset $\hat u$ is selected where radius and direction are approximately independent; the radial tails of reference ($P$) and generated ($Q$) samples are fitted above $\hat u$ using the fitted radial survival functions $\widehat S_P,\widehat S_Q$; generated extremes are then rescaled radially using the quantile-matching map $T_R$ computed from the fitted functions, leaving their direction unchanged.}
    \label{fig:pipeline}
\end{figure}

\subsection{Multivariate regular variation regime: detect and model}

The \textbf{Detect} and \textbf{Model} steps are both based on the polar characterization of MRV.
We first recall the definition of MRV and its polar characterization, which is central to our method.
Let $\|\cdot\|$ specify the Euclidean norm and set $\mathbb S^{d-1}:=\{x:\|x\|_2=1\}$.
\begin{theorem}[Informal radial--angular characterization of MRV]
    Let $X \in \mathbb{R}^d$ be multivariate regularly varying with tail index
    $\alpha > 0$, and write $R := \|X\|,\Theta := X/\|X\| \in \mathbb{S}^{d-1}$.
    In the extreme regime $R>u$, as $u\to\infty$:
    \begin{enumerate}[leftmargin=1.1em,label=\arabic*.,itemsep=0.1em,topsep=0.1em]
        \item the radius and angle become asymptotically independent;
        \item the normalized radius $R/u$, conditional on $R>u$, converges to a
              $\operatorname{Pareto}(\alpha)$ distribution.
    \end{enumerate}
\end{theorem}
Thus, the tail-onset detection exploits asymptotic radial--angular independence \textit{(1)}, whereas the radial tail modeling exploits the Pareto limit of the normalized radius \textit{(2)}.

\paragraph{Detect: select tail onset $\widehat u$.}
To identify the onset of the extreme regime, we consider a grid of exceedance fractions $[f_{\min},f_{\max}]$, ranging from a small fraction of the most extreme observations to progressively larger fractions extending toward the bulk. For each $f_j\in [f_{\min},f_{\max}]$, we define the corresponding radial threshold
\begin{equation*}
    u_j = \widehat Q_R(1-f_j),
\end{equation*}
where $\widehat Q_R$ corresponds to the empirical radial quantile function over the previously selected fitting sample, and restrict attention to the exceedances $\{R>u_j\}$. As $f_j$ increases, the threshold $u_j$ decreases and progressively less extreme observations are included.
This increases the number of available samples but may reduce the accuracy of the asymptotic approximation underlying the theorem, requiring a balance between statistical precision and approximation bias.
At each candidate threshold, we test
\begin{equation*}
    \mathcal H_0:\ R \perp \Theta \mid R>u_j\quad\text{against}\quad
    \mathcal H_1:\ R \not\!\perp \Theta \mid R>u_j,
\end{equation*}
using the radial--angular independence test of \citet{wan2019threshold}, which yields an associated averaged $p$-value $\overline p_j$.
This averaging is performed over repeated subsamples of the exceedances at threshold $u_j$ to reduce the variance of the permutation test.
We refer to~\autoref{sec:supp-exp-details} for the full fitting protocol and the averaging formula.
We then identify the point at which the sequence $(\overline p_j)_j$ moves from values compatible with the null hypothesis $\mathcal H_0$ of radial--angular independence to values providing evidence against it.
Applying this procedure separately to $P$, the real underlying distribution of our data, and $Q$, the trained model, yields onset estimates $\widehat u_P$ and $\widehat u_Q$.
We finally set $\widehat u=\max\{\widehat u_P,\widehat u_Q\}$, so that both distributions are in their respective extreme-tail regimes above the common threshold.

\paragraph{Model: produce radial tail fits $\widehat S_P$ and $\widehat S_Q$.}
Above the common threshold $\widehat u=\max\{\widehat u_P,\widehat u_Q\}$, the radial and angular components
are approximately decoupled.
Therefore, we are able to model the tail through the radius alone.
For each distribution $D\in\{P,Q\}$, we fit a Pareto tail to the exceedances of $R$ above $\widehat u$, with tail index estimated by the Hill estimator \citep{hill1975simple},
\begin{equation*}
    \widehat \alpha_D =\bigg[\frac{1}{N_D(\widehat u)}\sum_{R_i>\widehat u}\log\Big(\frac{R_i}{\widehat u}\Big)\bigg]^{-1},
\end{equation*}
where $N_D(\widehat u)$ counts the radii exceeding $\widehat u$.
For each $D\in\{P,Q\}$, the fitted radial survival function is
\[
    \widehat S_D(r) =
    \widehat\pi_D \left(\frac{r}{\widehat u}\right)^{-\widehat\alpha_D}, \qquad r\geq\widehat u,
\]
where $\widehat\pi_D=N_D(\widehat u)/n_D$ is the empirical exceedance probability and $n_D$ is the total fitting sample size.
The factor $\widehat\pi_D$ ensures that $\widehat S_D(r)$ estimates the unconditional probability of exceeding $r$.
The fitted survival functions $\widehat S_P$ and $\widehat S_Q$ are then used to construct the radial correction.

\subsection{Repair: apply radial correction map \texorpdfstring{$T_R$}{TR}.}
The previous steps provide a common tail threshold $\widehat u$ together with Pareto fits $\widehat S_P$ and $\widehat S_Q$ for the radial tails of the reference and generated distributions.
We define the correction by
\[
    T(x)=
    \begin{cases}
        c(\|x\|)\,x, \quad & \text{ for } \|x\|>\widehat u,    \\
        x,                 & \text{ for } \|x\|\leq\widehat u,
    \end{cases}
\]
where, for $r>\widehat u$,
\begin{equation*}
    c(r)=\exp\left(\frac{g(r)}{\widehat\alpha_P}\left[ \log \widehat S_P(r)-\log \widehat S_Q(r)\right]_+ \right).
\end{equation*}

The gate $g:[0,\infty)\to[0,1]$ gradually activates the
correction between $\widehat u$ and $\kappa\widehat u$ (with \(\kappa=1.25\) fixed throughout the paper).
\[
    g(r)=
    \begin{cases}
        0,               & \text{for } r\leq\widehat u,               \\
        3z(r)^2-2z(r)^3, & \text{for } \widehat u<r<\kappa\widehat u, \\
        1,               & \text{for } r\geq\kappa\widehat u,
    \end{cases}
\]
with
\[
    z(r) = \frac{r-\widehat u}{(\kappa-1)\widehat u}.
\]
Thus, the correction leaves the bulk unchanged, preserves directions, and expands radii only where the fitted generated survival is smaller than its reference counterpart.

Because radius and angle are approximately decoupled in the extreme selected regime, this transformation can correct the radial tail without altering the angular structure.
Its purpose is to compensate for tail undercoverage of the generator: when $\widehat S_Q(r)<\widehat S_P(r)$, the radius is pushed outward so that, in the fully active regime,
\begin{equation*}
    \widehat S_P\!\left(T_R(r)\right)
    =
    \widehat S_Q(r).
\end{equation*}
Hence, the correction restores the appropriate frequency and magnitude of extreme events through radial tail matching, while preserving their directional structure.

\section{Theoretical justification}
\label{sec:theory}

From a theoretical perspective, we see that this methodology is grounded in two different levels.
At the \emph{population level}, the radial correction aligns the tail index of $Q$ with that of $P$ while preserving the spectral measure (cf.\ \Cref{prop:repair}).
At the \emph{statistical level}, we can see the consistency of the plug-in correction based on Pareto fits above a common threshold, with explicit rates highlighting the role of the threshold (\autoref{thm:consistency}).
Proofs are deferred to the \autoref{app:theory}.

Throughout, for $D\in\{P,Q\}$, we have that $R_D=\|X_D\|$ and $\Theta_D=X_D/\|X_D\|$ denote the radius and direction respectively under $D$.
Moreover, $S_D(r)=\mathbb P(R_D>r)$ is the radial survival function, $S^{-1}(p)=\inf\{y\geq0:\ S(y)\leq p\}$ the generalized inverse of a nonincreasing function $S$, and $T_R = S_P^{-1}\circ S_Q$ the population radial transport.

\subsection{Population level: radial repair under MRV}

We now give the MRV repair. We denote the pushforward $T_{\#}Q$ as the law of $T(X_Q)$ for $X_Q\sim Q$.

\begin{proposition}[MRV repair]\label{prop:repair}
    Let $P$ and $Q$ be multivariate regularly varying on $\R^d$ with tail indices
    $\alpha_P,\alpha_Q>0$ and spectral measures $H_P,H_Q$ (\autoref{th:MRV}) \ie, the limiting conditional distributions of $\Theta_D$ given $R_D>u$ as $u\to\infty$.
    Assume that $S_P$ and $S_Q$ are continuous and strictly decreasing on $[u_0,\infty)$ for some $u_0>0$. For $u\geq u_0$, define
    \begin{equation*}
        T(x)=
        \begin{cases}
            T_R(\|x\|)\,\dfrac{x}{\|x\|}, \quad & \text{ for } \|x\|>u,     \\[2mm]
            x,                                  & \text{ for } \|x\|\leq u,
        \end{cases}
    \end{equation*}
    and let $Q^T=T_{\#}Q$.
    Then:
    \begin{enumerate}[leftmargin=1.1em,label=\arabic*.,itemsep=0.1em,topsep=0.1em]
        \item[(i)] \emph{Radial tail matching:} $\mathbb P(\|T(X_Q)\|>r)=S_P(r)$ for every
              $r\geq\max\{u,T_R(u)\}$;
        \item[(ii)] \emph{Angular preservation:} $T(X_Q)/\|T(X_Q)\|=\Theta_Q$
              almost surely;
        \item[(iii)] \emph{MRV repair:} $Q^T$ is multivariate regularly varying with tail
              index $\alpha_P$ and spectral measure $H_Q$, \ie, as $v\to\infty$,
              \begin{equation*}
                  \mathcal L\left(\frac{R_{Q^T}}{v},\Theta_{Q^T}\middle|R_{Q^T}>v\right)
                  \stackrel{\mathcal{L}}{\longrightarrow}
                  \operatorname{Pareto}(\alpha_P)\otimes H_Q.
              \end{equation*}
    \end{enumerate}
    In particular, if $H_Q=H_P$, then $Q^T$ and $P$ have the same asymptotic polar law.
\end{proposition}

\begin{proof}
    The proof is deferred to \autoref{app:repair_proof}.
\end{proof}

\autoref{prop:repair} clarifies what the correction can reach: it replaces $\alpha_Q$ by $\alpha_P$, while preserving the spectral measure $H_Q$ learned by the generator.
Any discrepancy between $H_Q$ and the target spectral measure $H_P$ therefore remains unchanged.
It also explains why the \textbf{Detect} step is needed: (iii) is asymptotic, and below the onset of the MRV regime the radius may depend on the direction, in which case rescaling radii along fixed rays would distort the joint law of $(R_Q,\Theta_Q)$.

\subsection{Statistical level: consistency of the estimated correction}

In practice, $T_R$ is replaced by a plug-in estimate.
For each $D\in\{P,Q\}$, let $R_1^D,\dots,R_{n_D}^D$ be i.i.d.\ copies of the radius $R_D=\|X_D\|$, where $X_D\sim D$ and $n_D=n_D(n)$ denotes the sample size indexed by $n$.
Given a common threshold $u_n$, with all asymptotic statements understood as $n\to\infty$, let $N_D=\sum_i\mathbf 1_{\{R^D_i>u_n\}}$ be the number of exceedances, $\widehat\pi_D=N_D/n_D$, $\widehat\alpha_D$ the Hill estimator above $u_n$, and $\widehat S_D(r)=\widehat\pi_D(r/u_n)^{-\widehat\alpha_D}$, $r\geq u_n$, the fitted Pareto tails.
The estimated transport is defined for $r\geq u_n$ by
\begin{equation*}
    \widehat T_R(r) = \widehat S_P^{-1}\big(\widehat S_Q(r)\big) = u_n\left(\frac{\widehat\pi_P}{\widehat\pi_Q}\right)^{1/\widehat\alpha_P} \left(\frac{r}{u_n}\right)^{\widehat\alpha_Q/\widehat\alpha_P},
\end{equation*}
which corresponds to the empirical counterpart of the explicit Pareto map of \autoref{subsec:model}.
We measure its error on the logarithmic scale, which is the natural one for scale-free tails: $\Delta_n(r)=\log(\widehat T_R(r))-\log (T_R(r))$.

We assume that the radial tails are asymptotically Pareto with a polynomial rate (second-order condition of Hall type):
\begin{enumerate}[leftmargin=1.4em,label=\arabic*.,itemsep=0.4em,topsep=0.3em]
    \item[(H)] For $D\in\{P,Q\}$,
          $S_D(r)=c_D\,r^{-\alpha_D}\big(1+\eta_D(r)\big)$ with $c_D>0$ and
          $|\eta_D(r)|\leq Kr^{-\rho_D}$ for $r\geq r_0$, for some $K,\rho_D>0$.
\end{enumerate}
This holds, for instance, for Student radii ($\rho_D=2$) and Fréchet radii ($\rho_D=\alpha_D$).
The quantity $m_D=n_D\,S_D(u_n)$ is the expected number of exceedances of $u_n$ under $D$.

Under this assumption, the following theorem provides a uniform error bound for the estimated radial correction and establishes its relative consistency over an expanding range of radii.

\begin{theorem}[Consistency of the estimated radial correction]
    \label{thm:consistency}
    Suppose that (H) holds. Moreover, assume that the threshold $u_n$ is deterministic with $u_n\to\infty$, and that $m_P,m_Q\to\infty$. Let $\rho=\min(\rho_P,\rho_Q)$, $\rho_\star=\min\{\rho_Q,\ \rho_P\min(1,\alpha_Q/\alpha_P)\}$, and let $r_n\geq u_n$. Then,
    \begin{multline}\label{eq:rate}
        \sup_{u_n\leq r\leq r_n}|\Delta_n(r)| =\mathrm O_P\Big(\big(m_P^{-1/2}+m_Q^{-1/2}+u_n^{-\rho}\big)\\\times(1+\log r_n)\Big)+\mathrm O\big(u_n^{-\rho_\star}\big).
    \end{multline}
    Consequently, $\sup_{u_n\leq r\leq r_n}\big|\widehat T_R(r)/T_R(r)-1\big|\overset{P}{\longrightarrow}0$ as soon as $(m_P^{-1/2}+m_Q^{-1/2}+u_n^{-\rho})\log r_n\to0$.
\end{theorem}

The previous bound separates three sources of error: the \emph{sampling} error $m_P^{-1/2}+m_Q^{-1/2}$ on the Pareto parameters, the \emph{Hill bias} $u_n^{-\rho}$, and the \emph{model misspecification} $u_n^{-\rho_\star}$ due to fitting an exact Pareto tail to an only asymptotically Pareto one.
The factor $(1+\log r_n)$ is the price of extrapolation: $\widehat T_R(r)$ evaluates the fitted quantile of $P$ at the level $\widehat S_Q(r)$, which may lie far beyond the observed exceedances, amplifying tail-index estimation errors through the logarithmic factor characteristic of Weissman extrapolation \citep{weissman1978estimation,deHaan2006extreme}.
When the tails are exactly Pareto above a fixed threshold $u$, the two bias terms vanish and the rate becomes parametric, $O_P\big(n^{-1/2}(1+\log(r_n/u))\big)$ with $n=\min(n_P,n_Q)$ (\Cref{prop:pareto} in \autoref{app:theory}).

From this analysis, we make the following observations:
\begin{enumerate}[leftmargin=1.4em,label=\arabic*.,itemsep=0.4em,topsep=0.3em]
    \item \emph{Threshold choice and the role of} \textbf{Detect}: The sampling term decreases when $u_n$ decreases (more exceedances), whereas both bias terms decrease when $u_n$ increases: this is the classical bias--variance trade-off of extreme-value statistics. If $n_P\asymp n_Q\asymp n$, then $m_D\asymp n u_n^{-\alpha_D}$, and balancing the two contributions with $\bar\alpha=\max(\alpha_P,\alpha_Q)$ gives $u_n\asymp n^{1/(\bar\alpha+2\rho_\star)}$ and the Hall-type rate $O_P\!\left(n^{-\rho_\star/(\bar\alpha+2\rho_\star)}(1+\log r_n)\right)$. Since $\rho_\star$ is unknown, the threshold is instead chosen by \textbf{Detect}, which targets the lowest level above which both $P$ and $Q$ are in their MRV regime.

    \item \emph{From the estimated transport to the implemented map}: On the region where the weight $g$ equals one, the map of the \textbf{Repair} step is $r\mapsto c(r)\,r=\max\{r,\widehat T_R(r)\}$: indeed $\log r+\widehat\alpha_P^{-1}\big(\log\widehat S_P(r)-\log\widehat S_Q(r)\big) =\log\widehat T_R(r)$, and the positive part activates the correction only when $\widehat S_Q(r)<\widehat S_P(r)$, \ie, when $\widehat T_R(r)>r$. Since $x\mapsto\max\{\log r,x\}$ is $1$-Lipschitz, \autoref{thm:consistency} transfers verbatim to the implemented map, with target $\max\{r,T_R(r)\}$, which coincides with $T_R$ in the undercoverage regime targeted by the method.

    \item\label{rem:threshold} \emph{Data-driven threshold}: \autoref{thm:consistency} is stated for a deterministic threshold. If $\widehat u$ is selected on an independent subsample (sample splitting), it holds conditionally on $\widehat u$, provided $m_D\to\infty$ along the selected thresholds with probability tending to one. A guarantee for the threshold produced by the radial--angular scan itself, whose successive tests are strongly dependent, is left for future work.
\end{enumerate}

\section{Experiments}

We evaluate Radial Tail Repair (RTR) on synthetic heavy-tailed and real-world datasets, comparing frozen generators before and after correction against heavy-tailed diffusion baselines.
Experiments are implemented in PyTorch and run on NVIDIA V100 GPUs.

\subsection{Frozen reference models}

For each dataset, we first train the generative baselines independently and freeze their parameters before applying RTR.
We consider DDPM~\citep{ho2020denoising}, DDIM~\citep{song2021denoising}, DLPM~\citep{shariatian2025dlpm}, and \(t\)-EDM~\citep{pandey2025heavytailed}, with DDIM reusing the corresponding DDPM checkpoint rather than being trained separately.
Full details on model architectures, training procedures, sampling configurations, and hyperparameter selection are provided in~\autoref{sec:supp-exp-details}.

\subsection{Experimental settings}

\paragraph{Datasets.}
We evaluate on two synthetic and three real-data benchmarks (\autoref{tab:experiment-settings}).
The synthetic datasets are a \(150\)-dimensional isotropic symmetric \(\alpha\)-stable distribution with \(\alpha=1.7\), and an imbalanced mixture of \(16\) stable modes with the same \(\alpha\) and dimension.
HRRR east and west contain hourly precipitation measurements at \(150\) cities per region, derived from the High-Resolution Rapid Refresh (HRRR) system~\citep{dowell2022hrrr}.
The wildfire dataset contains positive final fire sizes, in acres, from the USDA Fire Program Analysis Fire-Occurrence Database (FPA-FOD)~\citep{short2026fpa}, spanning 1992--2024.
Real data are randomly split into training, validation, and held-out evaluation sets.

\begin{table}[t]
    \centering
    \small
    \resizebox{0.8\linewidth}{!}{%
        \begin{tabular}{@{}lrrrr@{}}
            \toprule
            Dataset     & Dim. & Train/Val. & Eval. & \(p_{\min}\) \\
            \midrule
            Iso. stable & 150 & 25k/1k & 5M  & \(10^{-5}\) \\
            Stable mix. & 150 & 25k/1k & 5M  & \(10^{-5}\) \\
            HRRR east   & 150 & 25k/1k & 65k & \(10^{-3}\) \\
            HRRR west   & 150 & 25k/1k & 65k & \(10^{-3}\) \\
            Wildfire    & 1   & 25k/1k & 1M  & \(10^{-4}\) \\
            \bottomrule
        \end{tabular}%
    }
    \caption{Benchmark datasets and evaluation budgets. Tail probabilities are evaluated at \(21\) log-spaced points from \(10^{-1}\) to \(p_{\min}\).}
    \label{tab:experiment-settings}
\end{table}

\paragraph{Protocol.}
We compare DDPM, DDIM, DLPM, and \(t\)-EDM, together with zero-network DLPM and \(t\)-EDM controls that isolate tail behavior induced by the prior and sampler alone.
For each corrected model, RTR is applied to exactly the same generated samples as the uncorrected baseline, isolating the effect of the post-training correction.
Synthetic evaluation uses fresh reference draws, while real-data evaluation uses held-out samples.
A test--test baseline compares two disjoint halves of the reference evaluation sample.

\paragraph{RTR fitting.}
RTR fits the reference tail from \(5\mathrm{k}\) samples and the generated tail from \(5\mathrm{k}\) DDPM or DDIM samples.
We scan \(30\) candidate tail fractions from roughly \(6\%\) to \(30\%\), requiring at least \(N_{\min}=300\) exceedances.
The onset is selected using a radial--angular independence diagnostic based on the mean of repeated permutation-test \(p\)-values, with cutoff \(0.4\), after which a Hill/Pareto exponent is fitted.
The smooth gate then applies positive radial expansion with \(\kappa = 1.25\).

\paragraph{Metrics.}
In the experiments, we assess tail coverage with TCE and global distributional agreement with MMD.
Let \(X=\{x_i\}_{i=1}^{N_P}\) and \(Y=\{y_i\}_{i=1}^{N_Q}\) denote the reference and generated evaluation samples.
We evaluate tail coverage over \(21\) log-spaced probabilities \(p\) from \(10^{-1}\) to a dataset-specific \(p_{\min}\).
For coordinate \(j\), let \(t_j(p)\) denote the reference \((1-p)\)-quantile of \(F_j(x)\), where \(F_j(x)=|x_j|\) for synthetic data and \(F_j(x)=x_j\) for tabular data. We define
\begin{align}
    \widehat S_{D,j}(p) &= \frac{1}{|D|}\sum_{z\in D}\mathbf 1\{F_j(z)>t_j(p)\}, \quad D\in\{X,Y\}, \\
    \mathrm{TCE}(p) &= \frac{1}{d}\sum_{j=1}^d \left|\log\frac{\widehat S_{Y,j}(p)+\varepsilon}{\widehat S_{X,j}(p)+\varepsilon}\right|.
\end{align}
Here, \(\varepsilon=10^{-12}\), and lower TCE indicates better tail coverage.
We additionally report radial and maximum-event variants of TCE, their integrated areas, the signed maximum-event log-mass ratio, and RBF-MMD\(^{2}\).
MMD uses at most \(1\mathrm{k}\) samples for synthetic and wildfire data and \(512\) for HRRR.
Full definitions are given in~\autoref{sec:supp-metrics}.

\subsection{Results: Synthetic data}

The isotropic $\alpha$-stable case isolates heavy-tail behavior, while the 150-dimensional imbalanced stable mixture adds multimodal structure.
For each case, the curve resolves coordinate TCE across the quantile grid and the area barplot summarizes it by one value.
The MMD$^2$ plots provide a bulk-sensitive complement to these tail-focused metrics.

\begin{figure}[t]
    \centering
    \captionsetup[subfigure]{labelformat=empty,justification=centering}
    \begin{subfigure}[t]{0.5\linewidth}
        \centering
        \includegraphics[width=\linewidth]{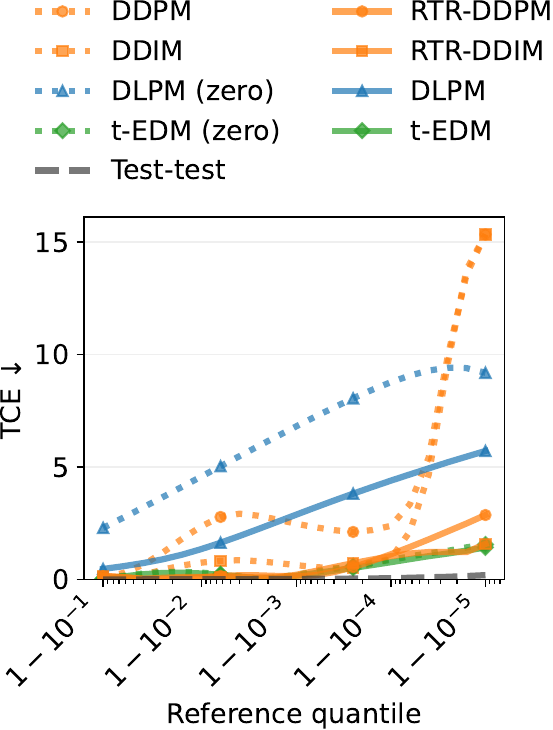}
        \caption{Isotropic $\alpha$-stable.}
        \label{fig:iso-tce}
    \end{subfigure}\hfill
    \begin{subfigure}[t]{0.48\linewidth}
        \centering
        \includegraphics[width=\linewidth]{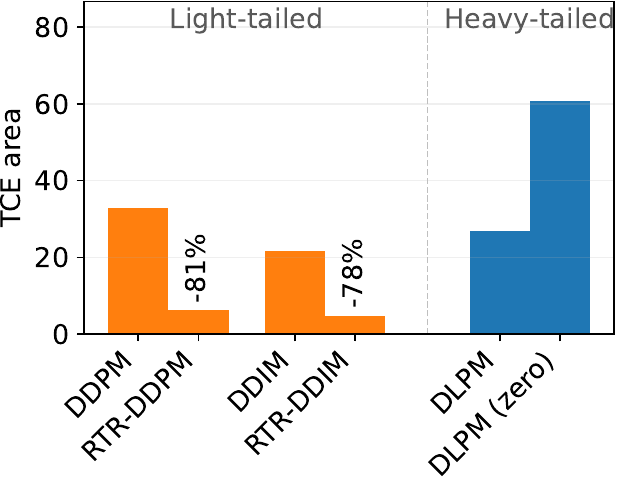}
        \caption{Isotropic $\alpha$-stable.}
        \label{fig:iso-area}
    \end{subfigure}
    \caption{Coordinate tail-coverage results for the 150-dimensional isotropic $\alpha$-stable distribution; lower TCE is better.}
    \label{fig:iso-results}
\end{figure}

\begin{figure}[t]
    \centering
    \captionsetup[subfigure]{labelformat=empty,justification=centering}
    \begin{subfigure}[t]{0.5\linewidth}
        \centering
        \includegraphics[width=\linewidth,trim={0mm 0mm 0mm 35mm},clip]{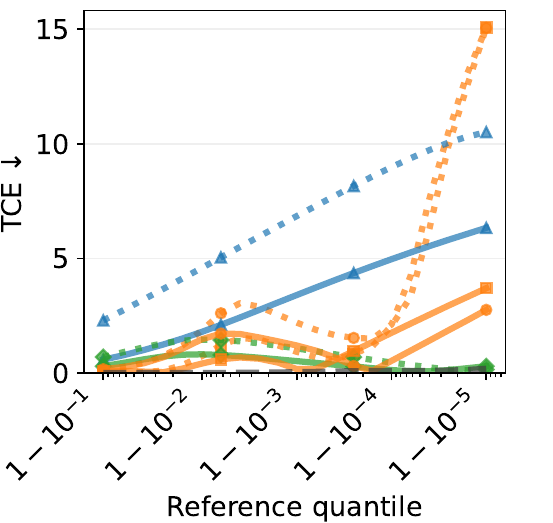}
        \caption{Stable mixture.}
        \label{fig:mixture-tce}
    \end{subfigure}\hfill
    \begin{subfigure}[t]{0.48\linewidth}
        \centering
        \includegraphics[width=\linewidth]{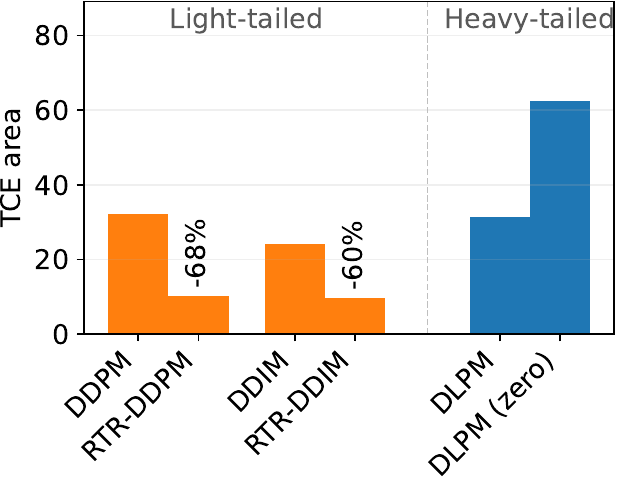}
        \caption{Stable mixture.}
        \label{fig:mixture-area}
    \end{subfigure}
    \caption{Coordinate tail-coverage results for the 150-dimensional imbalanced $\alpha$-stable mixture; lower TCE is better.}
    \label{fig:mixture-results}
\end{figure}

As shown in~\autoref{fig:iso-results} and~\autoref{fig:mixture-results}, the light-tailed DDPM/DDIM baselines remain reasonably accurate up to about the \(1-10^{-4}\) quantile, but increasingly underrepresent the far tail beyond that point.
DLPM departs from its zero-network control, indicating that it learns some tail structure, although its tail coverage remains inaccurate.
By contrast, trained \(t\)-EDM is nearly indistinguishable from its zero-network counterpart in the far tail, suggesting that little tail structure is learned there and that the apparent coverage mainly comes from the heavy-tailed prior and sampler.
RTR substantially improves the tail adequacy of the light-tailed models, showing that a post-training correction can recover much of the missing extreme-tail mass without changing the learned generator.

Because trained \(t\)-EDM is essentially indistinguishable from its zero-network control, we omit it from the barplots.
This lack of learned structure is itself disqualifying for a generative model, even if TCE and MMD may partly obscure it by focusing on tail mass and global discrepancy.

\begin{figure}[t]
    \centering
    \captionsetup[subfigure]{labelformat=empty,justification=centering}
    \begin{subfigure}[t]{0.46\linewidth}
        \centering
        \includegraphics[width=\linewidth]{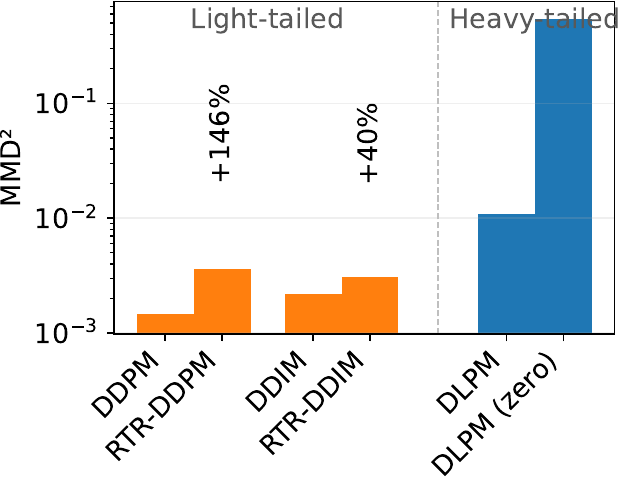}
        \caption{Isotropic $\alpha$-stable.}
        \label{fig:iso-mmd}
    \end{subfigure}\hfill
    \begin{subfigure}[t]{0.46\linewidth}
        \centering
        \includegraphics[width=\linewidth]{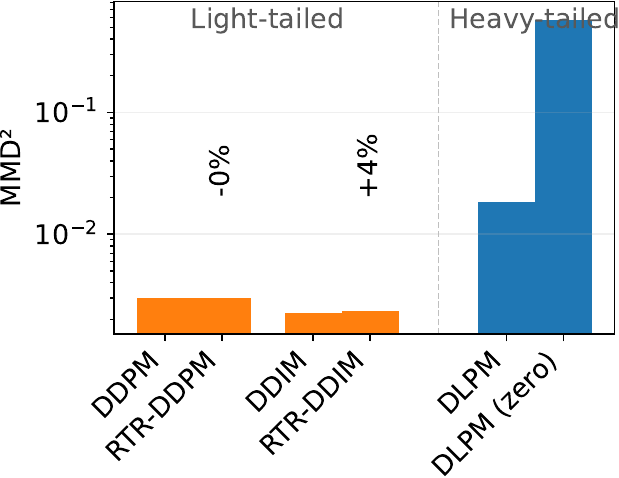}
        \caption{Stable mixture.}
        \label{fig:mixture-mmd}
    \end{subfigure}
    \caption{MMD$^2$ values on the synthetic datasets; lower values indicate closer distributional agreement.}
    \label{fig:synthetic-mmd}
\end{figure}

As shown in \autoref{fig:synthetic-mmd}, RTR introduces only a small absolute change in MMD compared with the much larger discrepancy of DLPM, suggesting that the bulk distribution is largely preserved.
The remaining MMD shift is consistent with the correction of extreme radii: although only tail samples are modified, their displacement can still affect this global metric.
We discuss this point further in \autoref{sec:discussion}.

\subsection{Results: Real data}

We next turn to real data to assess whether the behavior observed in the synthetic benchmarks extends to more realistic settings.

\begin{figure}[t]
    \centering
    \captionsetup[subfigure]{labelformat=empty,justification=centering}
    \includegraphics[width=0.49\linewidth,trim={0mm 90mm 0mm 0mm},clip]{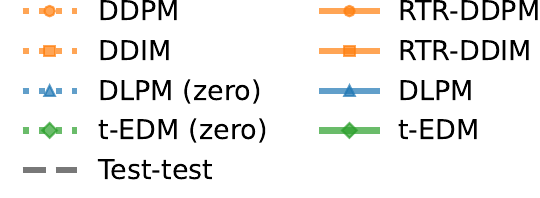}
    \par\medskip
    \begin{subfigure}[t]{0.49\linewidth}
        \centering
        \includegraphics[width=\linewidth,trim={0mm 0mm 0mm 35mm},clip]{figures/23__hrrr_east_hourly_precipitation_150d__coordinate_tce__curve.pdf}
        \caption{HRRR east precipitation.}
        \label{fig:hrrr-east-tce}
    \end{subfigure}\hfill
    \begin{subfigure}[t]{0.49\linewidth}
        \centering
        \includegraphics[width=\linewidth,trim={0mm 0mm 0mm 35mm},clip]{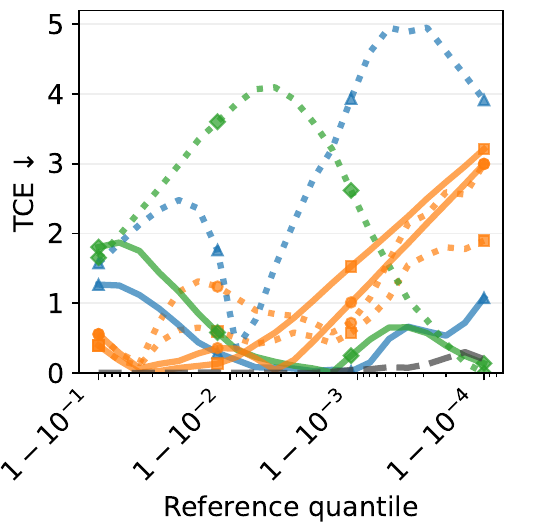}
        \caption{US wildfire size.}
        \label{fig:wildfire-tce}
    \end{subfigure}
    \caption{Coordinate tail-coverage error on real datasets.}
    \label{fig:real-tce}
\end{figure}

As shown in~\autoref{fig:real-tce}, HRRR exhibits the same qualitative behavior as the synthetic benchmarks: the heavy-tailed models fail to learn reliable tail structure, with DLPM now also remaining close to its zero-network control, while RTR substantially corrects the tail of the light-tailed models.
The wildfire case is particularly informative. In one dimension with the same \(25\mathrm{k}\) training samples, both DLPM and \(t\)-EDM clearly depart from their zero-network controls and recover meaningful tail behavior.
This suggests that their previous failures are driven less by model design than by finite-sample estimation in higher dimension, consistent with~\citet{cherkaoui2026heavytails}.
Conversely, wildfire exposes the approximation error of RTR: fitting an asymptotic Pareto tail can lead to residual misspecification, including far-tail overcorrection (see~\autoref{thm:consistency}).

Overall, RTR improves far-tail coverage across synthetic and real sparse-tail settings while largely preserving the bulk, with its extreme-value inductive bias enabling effective correction from scarce tail samples.
Additional results and diagnostics are provided in~\autoref{sec:supp-exp-additional}.

\section{Discussion}
\label{sec:discussion}

The experiments clarify both where RTR helps and what a radial post-training correction can reach.

\paragraph{Finite-sample learning of the tail.}
Our results reproduce the initialization--training trade-off highlighted in~\citet{cherkaoui2026heavytails}: accurate tail reconstruction depends critically on the amount of informative tail data available for training.
In the one-dimensional wildfire setting, the same training budget provides substantially more information about the tail structure, and the trained model separates more clearly from its zero-network control.
In the higher-dimensional settings, this separation largely disappears.

\paragraph{Tail mass is not the same as learned structure.}
The \(t\)-EDM results also show that good tail coverage can arise largely from inductive bias.
In the far tail, the trained model closely tracks its zero-network counterpart, suggesting that the prior and sampler can produce well-calibrated tail mass while learning little of the underlying structure.
This highlights a limitation of TCE: it measures extreme probability mass, but not whether the corresponding structure has been learned.
A heavy-tailed prior may therefore score well on TCE while remaining a poor generative model.

\paragraph{Finite-sample correction remains an extreme-value estimation problem.}
\autoref{thm:consistency} highlights the main trade-off in RTR: lower thresholds provide more exceedances, while higher thresholds reduce Pareto approximation bias, and extrapolation amplifies tail-index error.
This reflects a central challenge of extreme value theory: inference targets rare events, for which finite samples necessarily provide limited information, requiring a balance between the accuracy of the asymptotic approximation and the number of observations available for estimation.
Threshold selection is therefore a critical component of the method. Our current \textbf{Detect} step selects it through a data-driven scan, but more advanced threshold-selection strategies could substantially improve the stability and accuracy of RTR.

\paragraph{Bulk preservation is difficult to assess with global metrics.}
By construction, RTR leaves samples below the selected onset unchanged, preserving the learned bulk at the sample level.
Standard global metrics do not directly certify this: RBF-MMD\(^{2}\) mixes bulk and tail discrepancies and may react strongly to moving a small number of extreme samples.
On the isotropic stable benchmark, RTR greatly improves tail coverage while increasing MMD\(^{2}\), whereas on HRRR large TCE gains come with almost no MMD\(^{2}\) change.
MMD\(^{2}\) is therefore a useful global sanity check, but not a direct measure of bulk preservation.

\section{Conclusion}
\label{sec:conclusion}

In this paper, we asked whether accurate tail coverage can be achieved through lightweight post-hoc correction without costly retraining.
We proposed \emph{Radial Tail Repair (RTR)}, a smoothly gated radial rescaling that corrects undercovered tails while preserving directions and leaving the bulk unchanged.
Under multivariate regular variation, we characterize the population repair and establish consistency of its estimate, with explicit rates highlighting threshold choice and extrapolation.
Across challenging synthetic distributions and real precipitation and wildfire data, RTR improves tail coverage from few correction samples, including when extremes are scarce.
This allows existing generators to be retained and their tails corrected \textit{a posteriori}, providing a complementary alternative to heavy-tailed generative models that may require substantially more tail data.

More broadly, our results suggest a simple division of labor between generative modeling and extreme value theory: \textbf{expressive models can learn the data-rich bulk, while principled extreme-value methods handle the tail, where observations are inherently scarce and purely data-driven learning becomes unreliable.}

%------------------------------------------------------------------------------------------------------------------
% STATEMENT & BIBLIOGRAPHY
\clearpage

\section*{Disclosure of Funding}
The work of Hamza Cherkaoui and H\'el\`ene Halconruy was supported by the French National Research Agency (ANR) under grant ANR-24-CE40-3341 (project DECATTLON).
The work of Antonio Ocello was supported by Hi! PARIS and received government funding managed by the Agence Nationale de la Recherche under the France 2030 program, references (ANR-23-IACL-0005) and it was partially supported by the Labex Ecodec (reference project ANR-11-LABEX-0047).
This work was granted access to the HPC resources of IDRIS under the allocation 2026-AD011016818R1, provided by GENCI.

% \section*{AI use statement}
% In this work, we did use generative AI tools for tasks requiring disclosure.
% %
% Generative AI tools were \textbf{not} used to develop theoretical or conceptual contributions, formulate mathematical claims or proofs, design the methodology or experiments, process datasets, or interpret the results.
% %
% We used generative AI tools only for \textbf{drafting and editing parts of the manuscript to improve readability, and code proofreading}.
% %
% We have reviewed all AI-assisted work.
% %
% \textbf{All AI-assisted text and code edits were reviewed by the authors}.
% %
% We take responsibility for the final content of this work, including text, claims, or artifacts produced with the aid of generative AI.

\bibliographystyle{apalike}
\bibliography{4_bibliography}

% \input{3_checks}

%------------------------------------------------------------------------------------------------------------------
% APPENDIX
\clearpage
\appendix
\thispagestyle{empty}
\onecolumn
\aistatstitle{Radial Tail Repair: Post-Training Correction for Diffusion and Flow Models: Supplementary Materials}

\section*{Supplementary Contents}

\begin{enumerate}[align=left,itemsep=0.5cm,label=\Alph*.]
  \item \autoref{sec:supp-notations}: notation \dotfill (p.~\pageref{sec:supp-notations}).
  \item \autoref{app:theory}: theoretical analysis \dotfill (p.~\pageref{app:theory}).
  \item \autoref{sec:supp-exp-details}: experimental details \dotfill (p.~\pageref{sec:supp-exp-details}).
  \item \autoref{sec:supp-exp-additional}: additional experiments \dotfill (p.~\pageref{sec:supp-exp-additional}).
  \item \autoref{sec:supp-metrics}: evaluation metrics \dotfill (p.~\pageref{sec:supp-metrics}).
  \item \autoref{appendix-related-work}: additional related work \dotfill (p.~\pageref{appendix-related-work}).
\end{enumerate}

\section{Notation}
\label{sec:supp-notations}

For a frozen sampler $s$, write $Q_s$ for its endpoint law and omit the subscript when discussing one sampler.
Hats denote empirical estimates; fitting and evaluation samples are specified separately.

\begin{center}
  \small
  \begin{tabular}{@{}p{0.27\linewidth}p{0.66\linewidth}@{}}
    \toprule
    Symbol & Meaning \\
    \midrule
    $P,Q_s,Q_s^{\mathrm{RTR}}$ & Reference law, uncorrected endpoint law, and law after radial tail repair. \\
    $X_D,R_D,\Theta_D$ & A draw under $D\in\{P,Q_s\}$, its radius $R_D=\|X_D\|_2$, and its direction $\Theta_D=X_D/R_D$ when $R_D>0$. \\
    $\alpha_D,\gamma_D,H_D$ & MRV tail index, extreme-value index $\gamma_D=1/\alpha_D$, and spectral measure of $D$. \\
    $S_D,S_D^{-1}$ & Radial survival $S_D(r)=\mathbb P_D(R_D>r)$ and generalized inverse $S_D^{-1}(p)=\inf\{r\geq0:S_D(r)\leq p\}$. \\
    $\widehat u_D,u$ & Estimated tail onset for $D$ and common correction onset $u=\max(\widehat u_P,\widehat u_{Q_s})$. \\
    $f_j,\overline p_j$ & Candidate exceedance fraction and mean radial--angular independence-test $p$-value at that fraction. \\
    $n_D,N_D(u),\widehat\pi_D$ & Sample size in the current calculation, number of radii exceeding $u$, and fitted exceedance fraction $\widehat\pi_D=N_D(u)/n_D$. \\
    $\widehat\gamma_D,\widehat\alpha_D,\widehat S_D$ & Hill extreme-value index $\widehat\gamma_D$, fitted tail index $\widehat\alpha_D=1/\widehat\gamma_D$, and fitted radial Pareto survival. \\
    $T_R,\widehat T_R,T,g$ & Population radial transport $T_R=S_P^{-1}\!\circ S_{Q_s}$, fitted radial map, endpoint correction, and smooth gate. \\
    $a_t,\mu$ & MRV scaling sequence and its homogeneous limit measure on $\mathbb R^d\setminus\{0\}$. \\
    $\beta,\widehat\beta$ & Population and fitted Pareto transport exponents; $\beta=\alpha_{Q_s}/\alpha_P$ and $\widehat\beta=\widehat\alpha_{Q_s}/\widehat\alpha_P$. \\
    $c_D,\eta_D,\rho_D$ & Second-order tail model $S_D(r)=c_Dr^{-\alpha_D}(1+\eta_D(r))$, with $|\eta_D(r)|=O(r^{-\rho_D})$. \\
    $m_D,\mu_D,\nu_D$ & Expected exceedance count $m_D=n_DS_D(u)$ and the first two conditional moments of $\log(R_D/u)$ given $R_D>u$. \\
    $B_D,\delta_D$ & Hill bias $B_D=\mu_D-\gamma_D$ and log-relative exceedance-mass error $\delta_D=\log(\widehat\pi_D/S_D(u))$. \\
    $T_n^\circ,\Delta_n(r),r_n$ & Oracle Pareto map, log error $\Delta_n(r)=\log\widehat T_R(r)-\log T_R(r)$, and upper radius for a uniform error bound. \\
    $p,F,t_F,L_F$ & Evaluation tail probability, scalar statistic, reference quantile threshold, and signed generated-to-reference log exceedance ratio. \\
    $\mathrm{TCE}_{*},\mathcal A_E$ & Tail-coverage error curve and its area over reference-quantile logit. \\
    $\widehat{\mathrm{MMD}}_{\mathrm b}^2,\gamma$ & Biased squared RBF-MMD and its reference-derived kernel bandwidth; $\gamma_D$ separately denotes the extreme-value index. \\
    \bottomrule
  \end{tabular}
\end{center}

\section{Theoretical analysis}\label{app:theory}

The theoretical analysis can be organized around the three stages of the method: \emph{Detect}, \emph{Model}, and \emph{Repair}.
The first two rely on the multivariate regular variation (MRV) structure, while the third is justified by a radial transport argument.

\subsection{Background}\label{app-subsec:background}

Multivariate regular variation is the cornerstone of the whole method.
It is the standard framework for heavy-tailed random vectors: 
\begin{definition}
    $X\in\R^d$ is said to be multivariate regularly varying (MRV) with tail index $\alpha>0$ if there exist a
    sequence $a_t\to\infty$ and a nonzero Radon measure $\mu$ on $\R^d\setminus\{0\}$ such that
    \begin{equation*}
        t\,\mathbb P\left(\frac{X}{a_t}\in\cdot\right)\longrightarrow\mu(\cdot) \qquad\text{vaguely, as } t\to\infty,
    \end{equation*}
    where the limit measure is necessarily homogeneous of order $-\alpha$: $\mu(sA)=s^{-\alpha}\mu(A)$ for all $s>0$. 
\end{definition}
In polar coordinates, MRV means that far in the tail the law of $X$ splits into two asymptotically independent ingredients (\autoref{fig:polar}): a radial part, Pareto with index $\alpha$, which tells \emph{how heavy} the tail is, and an angular part, a probability measure on $\mathbb S^{d-1}$, which tells \emph{in which directions} extremes occur.
This decoupling is what allows the tail to be corrected by acting on the radius alone.

\begin{figure}[ht]
    \centering
    \begin{tikzpicture}[scale=1]
        \draw[gray] (0,0) circle (1);
        \node[gray] at (-0.55,-1.15) {$\mathbb S^{d-1}$};
        \fill (0,0) circle (1pt) node[below left] {$0$};
        % angular mass H
        \draw[line width=3pt, blue!60, opacity=.7] (20:1) arc (20:50:1);
        \draw[line width=3pt, blue!60, opacity=.7] (195:1) arc (195:220:1);
        % extreme points, sparser as the radius grows
        \foreach \a/\r in {28/1.35,42/1.5,33/1.7,24/1.95,46/2.2,37/2.6,30/3.1}
            \fill[black!70] (\a:\r) circle (1.2pt);
        \foreach \a/\r in {200/1.4,212/1.6,205/1.9,216/2.3,203/2.8}
            \fill[black!70] (\a:\r) circle (1.2pt);
        % one point and its polar decomposition
        \draw[->, thick] (0,0) -- (36:3.7) node[above right] {$X$};
        \fill[red] (36:1) circle (1.8pt);
        \node[red] at (60:1.25) {$\Theta$};
        \draw[decorate, decoration={brace, amplitude=5pt, mirror}]
            (36:0.05) ++(-54:0.25) -- ++(36:3.6)
            node[midway, below right=3pt] {$R=\|X\|$};
        % legend
        \node[align=left, anchor=west] at (4.3,1.0)
            {$R/u\mid R>u\ \approx\ \operatorname{Pareto}(\alpha)$\\[1pt]
             \quad\emph{how heavy}};
        \node[align=left, anchor=west, blue!70!black] at (4.3,-0.4)
            {$\Theta\mid R>u\ \approx\ H$\\[1pt]
             \quad\emph{in which directions}};
        \node[anchor=west] at (4.3,-1.5) {$R/u\ \perp\!\!\!\perp\ \Theta$ \ asymptotically};
    \end{tikzpicture}
    \caption{Polar picture of MRV. Extreme points become sparser along the radius at a Pareto rate (how heavy) and concentrate in the directions charged by the spectral measure $H$ (in which directions).}
    \label{fig:polar}
\end{figure}
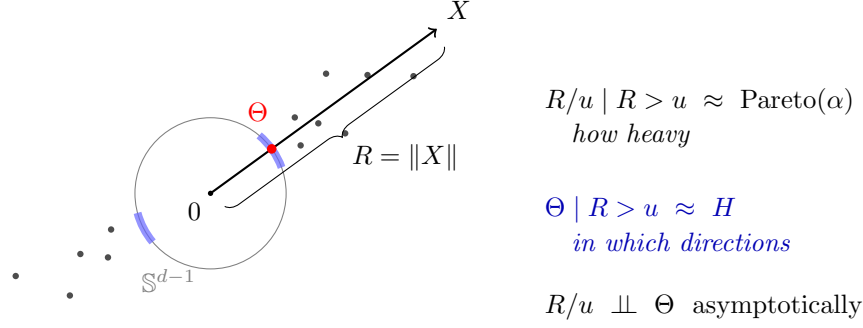

We recall the polar characterization of MRV.

\begin{theorem}\label{th:MRV}
    Let $X \in \mathbb{R}^d$ be multivariate regularly varying with tail index $\alpha > 0$, and write $R = \|X\|,\Theta=X/\|X\|\in\mathbb{S}^{d-1}$.
    Then there exists a probability measure $H$ on $\mathbb{S}^{d-1}$, called the spectral measure, such that, as $u \to \infty$,
    \begin{equation*}
        \mathcal{L}\left( \frac{R}{u}, \Theta \big | R > u \right) \Longrightarrow \operatorname{Pareto}(\alpha) \otimes H.
    \end{equation*}
    Equivalently, for every $r \geq 1$ and every Borel set $A \subset \mathbb{S}^{d-1}$ such that $H(\partial A)=0$,
    \begin{equation*}
        \mathbb{P}\left( \frac{R}{u} > r, \Theta \in A\middle|\, R > u\right)\longrightarrow r^{-\alpha} H(A).
    \end{equation*}
    Hence,
    \begin{equation*}
        \frac{R}{u}\big| R>u \Longrightarrow \operatorname{Pareto}(\alpha),
    \end{equation*}
    while the normalized radius and angular component become asymptotically independent:
    \begin{equation*}
        \frac{R}{u} \perp\!\!\!\perp \Theta \qquad \text{asymptotically, conditionally on } R>u.
    \end{equation*}   
\end{theorem}

\subsection{Justification of the method}
The method rests on the premise that, beyond some radius, both the reference distribution and the generated one are in their MRV regime, \ie in the region where the factorization of \autoref{th:MRV} is (approximately) in force.
The three arguments below explain why this regime is precisely the one in which a radial correction is legitimate, and hence why the first stage of the method must be to detect it:
\begin{enumerate}
    \item[(i)] any discrepancy between two radial laws can be removed by a monotone one-dimensional map (\emph{radial tail matching});
    \item[(ii)] applying such a map along rays leaves directions exactly unchanged (\emph{preservation of the angular structure});
    \item[(iii)] combining (i) and (ii) with the asymptotic radial--angular independence provided by MRV, the corrected law inherits the tail index of $P$ and keeps the spectral measure of $Q$ (\emph{MRV repair theorem}).
\end{enumerate}
Steps (i) and (ii) are elementary and hold for any distributions; it is step (iii), and therefore the MRV regime, that turns them into a sound correction of the tail.

\subsubsection{Proof of Proposition \autoref{prop:repair}}\label{app:repair_proof}

\paragraph{(i) Population radial tail matching}
Let $P$ and $Q$ denote the reference and generated distributions, respectively. For $D\in\{P,Q\}$, let
\begin{equation*}
    S_D(r) = \mathbb P(R^D>r)
\end{equation*}
denote the radial survival function in the tail regime, and assume that $S_D$ is positive and continuous on $[0,\infty)$. For any nonincreasing survival function $S$, we write
\begin{equation*}
    S^{-1}(p) = \inf\{y\geq 0:\ S(y)\leq p\},\qquad p\in(0,1),
\end{equation*}
for its generalized inverse. \\
We define the population radial transport by
\begin{equation*}
    T_R(r) = S_P^{-1}\left(S_Q(r)\right).
\end{equation*}
Since both $S_Q$ and $S_P^{-1}$ are nonincreasing, $T_R$ is nondecreasing.
Moreover, if $R_Q$ denotes the radius under $Q$, then $T_R(R_Q)$ has the same tail law as the radius under $P$, equivalently
\begin{equation*}
    S_P\circ T_R=S_Q.
\end{equation*}
Thus, at the population level, the correction amounts to a monotone one-dimensional transport that matches the radial tail law of $Q$ to that of $P$.
This step, however, only acts on the norm: it says nothing about what happens to the full vector, which is the object of the next paragraph.

\paragraph{(ii) Preservation of the angular structure}
For $x\neq 0$, write $x=r\theta$, $r=\|x\|$, $\theta=x/\|x\|$.
Consider the radial transformation
\begin{equation*}
    T(x)=T_R(r)\theta.
\end{equation*}
Then the angular component is preserved exactly: $\Theta_{T(X)}=\Theta_X$, so that $\mathcal L(\Theta_{T(X)})=\mathcal L(\Theta_X)$.
The correction therefore modifies tail magnitude while leaving the directional
structure learned by the generator unchanged.
Note, however, that preserving the marginal law of $\Theta$ is not enough to control the joint law of $(R,\Theta)$: if the radius and the direction are dependent, rescaling radii while keeping directions fixed distorts their joint behaviour in an uncontrolled way.
The radial correction is thus meaningful only in a region where $R$ and $\Theta$ are (approximately)
decoupled, which is exactly what MRV guarantees in the tail.

\paragraph{(iii) MRV repair theorem}
Assume that $P$ and $Q$ are multivariate regularly varying, and that their polar limit laws as $u\to\infty$ are characterized by \autoref{th:MRV}:
\begin{equation*}
    \mathcal L\left( \frac{R_P}{u},\Theta_P\middle | R_P>u\right) \Longrightarrow \operatorname{Pareto}(\alpha_P)\otimes H_P ,
\end{equation*}
and
\begin{equation*}
    \mathcal L\left( \frac{R_Q}{u},\Theta_Q\middle | R_Q>u\right) \Longrightarrow\operatorname{Pareto}(\alpha_Q)\otimes H_Q .
\end{equation*}
If the radial correction transforms the radial tail of $Q$ into that of $P$, then the corrected law $Q^T=T_{\#}Q$ satisfies
\begin{equation*}
    \mathcal L\left( \frac{R_{Q^T}}{u},\Theta_{Q^T} \,\middle|\, R_{Q^T}>u \right) \Longrightarrow \operatorname{Pareto}(\alpha_P)\otimes H_Q ,
\end{equation*}

Indeed, the identity $S_P\circ T_R=S_Q$ ensures that the transformed radius has the same survival function as the reference radius, and hence tail index $\alpha_P$.
Since the transformation preserves directions, the asymptotic radial--angular independence under $Q$
then yields the product limit $\operatorname{Pareto}(\alpha_P)\otimes H_Q$.
Thus, the correction replaces the asymptotic radial component of $Q$ by that of $P$, while preserving the spectral measure of $Q$.
In particular, if $H_Q=H_P$, then the corrected distribution has the same asymptotic polar law as $P$.

\paragraph{Consequence: detecting the MRV regime}
Steps (i) and (ii) are exact, but (iii) is asymptotic: it requires that, above some level $u$ and for both $D\in\{P,Q\}$, the radius $R_D$ be approximately Pareto and approximately independent of $\Theta_D$.
Below this level, $R_D$ is not of Pareto type and may depend on $\Theta_D$, so the radial map would
distort the bulk, which the generator is assumed to reproduce well.
Detecting the MRV regime therefore amounts to estimating a threshold $u$, \emph{common to $P$ and $Q$}, beyond which both are in their MRV regime; the correction is applied for $\|x\|>u$ and the identity is kept below (Figure~\ref{fig:detect}).

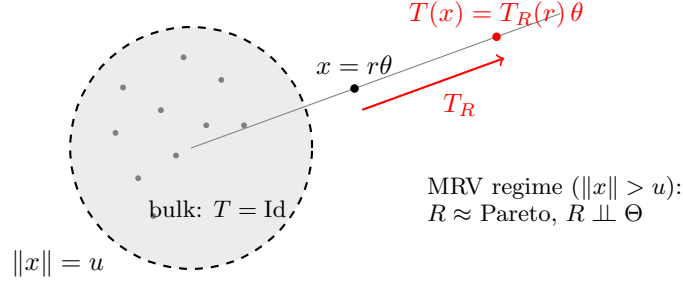
\begin{figure}[ht]
    \centering
    \begin{tikzpicture}[scale=1]
        \fill[gray!15] (0,0) circle (1.6);
        \draw[dashed, thick] (0,0) circle (1.6);
        \node[anchor=north east] at (-130:1.6) {$\|x\|=u$};
            \foreach \x/\y in {0.2/0.3,-0.4/0.5,-0.7/-0.4,-0.2/-0.1,
                               -1.0/0.2,0.4/0.9,-0.5/-0.9,-0.9/0.8,0.7/0.3,-0.1/1.2}
                \fill[black!50] (\x,\y) circle (1.1pt);
            \node[align=center, font=\small] at (0.35,-0.8) {bulk: $T=\mathrm{Id}$};
            \draw[gray] (0,0) -- (20:5.2);
            \fill (20:2.3) circle (1.6pt) node[above=2pt] {$x=r\theta$};
            \fill[red] (20:4.3) circle (1.6pt) node[above, red] {$T(x)=T_R(r)\,\theta$};
            \draw[->, red, thick] (20:2.3) ++(-70:0.3) -- ++(20:2)
                node[midway, below right] {$T_R$};
            \node[align=left, anchor=west, font=\small] at (3.0,-0.7)
                {MRV regime ($\|x\|>u$):\\ $R\approx$ Pareto, $R\perp\!\!\!\perp\Theta$};
    \end{tikzpicture}
    \caption{Detect and repair.
    Inside the ball of radius $u$, the bulk is left unchanged; beyond $u$, where both $P$ and $Q$ are in their MRV regime, each point is moved along its own ray by the radial map $T_R$.}
    \label{fig:detect}
\end{figure}

\subsection{Model: Explicit form in the Pareto case}\label{subsec:model}
Suppose that, above a common threshold $u$,
\begin{equation*}
    S_P(r) = \pi_P\left(\frac{r}{u}\right)^{-\alpha_P}\quad\text{and}\quad
    S_Q(r) = \pi_Q\left(\frac{r}{u}\right)^{-\alpha_Q}.
\end{equation*}
Define the population exceedance probabilities as $\pi_D=S_D(u)=\mathbb P(R_D>u)$ for $D\in\{P,Q\}$.
Note that these are distinct from their estimates $\widehat\pi_D$. Then,
\begin{equation*}
    T_R(r) = u \left[ \frac{\pi_P}{\pi_Q} \left(\frac{r}{u}\right)^{\alpha_Q} \right]^{1/\alpha_P}.
\end{equation*}
Equivalently, $T_R(r)=Cr^{\alpha_Q/\alpha_P}$, for a suitable constant $C$.
This makes the effect of the correction explicit: if $\alpha_Q>\alpha_P$, the generated tail is lighter and the map expands large radii; if $\alpha_Q=\alpha_P$, the correction reduces to a multiplicative rescaling.

In practice, the tail parameters $(\pi_P,\alpha_P)$ and $(\pi_Q,\alpha_Q)$ are unknown, and the \emph{Model} stage consists in estimating this map.
Given the threshold $u$ provided by the \emph{Detect} stage, we fit a Pareto tail to each sample above $u$, which yields estimated survival functions $\widehat S_P$ and $\widehat S_Q$, and define the plug-in map $\widehat T_R$ as the solution of the empirical counterpart of the matching equation $S_P\circ T_R=S_Q$, namely
\begin{equation*}
    \widehat S_P\circ\widehat T_R=\widehat S_Q \qquad\text{on } [u,\infty).
\end{equation*}
The quality of the correction is thus governed by how well the fitted tails approximate the true ones above $u$; this is the object of \autoref{app-subsec:consistency}.

\subsection{Consistency of the estimated correction}
\label{app-subsec:consistency}

The practical map is obtained by replacing the unknown survival functions with fitted Pareto tails above a common threshold.
We study it under two assumptions on the radial tails, which play complementary roles.
\begin{itemize}
    \item \emph{Exact Pareto tails above a fixed threshold} (\autoref{prop:pareto}).
    The model of \autoref{subsec:model} is then well specified and only the estimation error remains: the plug-in map is consistent at the parametric rate, and the cost of extrapolating far beyond $u$ is made explicit.
    \item \emph{Asymptotically Pareto tails under a second-order condition} (\autoref{thm:consistency}).
    The Pareto model is now only approximately correct above any fixed level, so the threshold must grow with $n$.
    The theorem quantifies the resulting bias--variance trade-off and gives explicit rates, which in turn guide the choice of $u$ (\autoref{rem:threshold}).
\end{itemize}
Throughout, $\gamma_D=1/\alpha_D$ denotes the extreme-value index of $D\in\{P,Q\}$ and
\begin{equation*}
    \beta=\frac{\alpha_Q}{\alpha_P}=\frac{\gamma_P}{\gamma_Q}
\end{equation*}
is the exponent of the population map in the Pareto case.

\subsubsection*{Setting}

For each $n$, we observe two i.i.d. samples of radii $R^P_1,\dots,R^P_{n_P}\sim S_P$ and $R^Q_1,\dots,R^Q_{n_Q}\sim S_Q$, with $n_P=n_P(n)$, $n_Q=n_Q(n)$.
The survival functions $S_D$ are continuous on $[0,\infty)$ and positive.
The threshold $u=u_n>0$ is deterministic and common to $P$ and $Q$ (see \autoref{rem:threshold} for a data-driven threshold).

\subsubsection*{Estimators}

For $D\in\{P,Q\}$, let $ N_D=\sum_{i=1}^{n_D}\mathbf 1_{\{R^D_i>u\}}$, $ \widehat\pi_D=N_D/n_D$ and the exponent estimators
\begin{equation*}
    \widehat\gamma_D = \frac{1}{N_D}\sum_{i=1}^{n_D}\log\Big(\frac{R^D_i}{u}\Big)\mathbf 1_{\{R^D_i>u\}}\quad\text{and}\quad \widehat\alpha_D =\frac1{\widehat\gamma_D},
\end{equation*}
on the event $E_n=\{N_P\geq1,\ N_Q\geq1\}$ (on which $\widehat\gamma_D>0$, since exceedances are strict). The fitted tails are $\widehat S_D(r)=\widehat\pi_D(r/u)^{-\widehat\alpha_D}$, $r\geq u$, and $\widehat S_P^{-1}(p)=u(\widehat\pi_P/p)^{\widehat\gamma_P}$.
Hence
\begin{equation}\label{eq:That}
    \widehat T_R(r) = \widehat S_P^{-1}\big(\widehat S_Q(r)\big) =
        u\left(\frac{\widehat\pi_P}{\widehat\pi_Q}\right)^{\widehat\gamma_P} \left(\frac{r}{u}\right)^{\widehat\beta}
        \quad\text{where}\quad
        \widehat\beta=\frac{\widehat\gamma_P}{\widehat\gamma_Q}=\frac{\widehat\alpha_Q}{\widehat\alpha_P},
        \qquad r\geq u .
\end{equation}
Off $E_n$, set $\widehat T_R(r)=r$ (any convention will do, since $\mathbb P(E_n)\to1$ below).

\subsubsection*{Auxiliary quantities}

For $D\in\{P,Q\}$, define the \emph{effective number of exceedances}, the conditional log-moments above $u$, the Hill bias and the relative error of $\widehat\pi_D$:
\begin{gather}\label{eq:auxiliary_quantities}
    m_D = n_D\,S_D(u), \qquad
    \mu_D = \E\Big[\log\Big(\frac{R^D}{u}\Big)\Big|\,R^D>u\Big], \qquad
    \nu_D = \E\Big[\log^2\Big(\frac{R^D}{u}\Big)\Big|\,R^D>u\Big], \nonumber \\
    B_D = \mu_D-\gamma_D, \qquad \delta_D = \log\Big(\frac{\widehat\pi_D}{S_D(u)}\Big).
\end{gather}
All of them depend on $n$ through $u=u_n$ and $n_D$.
Finally, define the \emph{oracle Pareto map} at level $u$, the log-ratio of exceedance probabilities, and the approximation bias:
\begin{equation*}
    T^\circ_n(r)= u\left(\frac{S_P(u)}{S_Q(u)}\right)^{\gamma_P}\left(\frac{r}{u}\right)^{\beta},
    \qquad L_n=\log\Big(\frac{S_P(u)}{S_Q(u)}\Big), \qquad b_n(r)=\log (T_R(r))-\log(T^\circ_n(r)).
\end{equation*}
$T^\circ_n$ is what $\widehat T_R$ would be if $(\widehat\pi_D,\widehat\gamma_D)$ were replaced by their targets $(S_D(u),\gamma_D)$ for $D\in\{P,Q\}$, and $b_n(r)$ is the bias of the Pareto model at level $u$.

\subsubsection*{Tail assumptions}
We will rely on the following assumptions on the radial tails, which are standard in extreme-value theory.
They hold for $D\in\{P,Q\}$.
\begin{enumerate}
    \item[(P)] \emph{Exact Pareto tail above a fixed threshold}: $u_n\equiv u$ and $S_D(r)=\pi_D(r/u)^{-\alpha_D}$ for $r\geq u$, with $\pi_D\in(0,1)$.
    \item[(H)] \emph{Second-order Pareto tail} (Hall type): $S_D(r)=c_D\,r^{-\alpha_D}\big(1+\eta_D(r)\big)$ with $c_D>0$ and $|\eta_D(r)|\leq K r^{-\rho_D}$ for $r\geq r_0$, with $K,\rho_D>0$.
\end{enumerate}
\emph{Under (P)}, the Pareto model is exact above the fixed threshold $u$, so the targets of the estimators do not depend on $n$: $S_D(u)=\pi_D$, and $\log(R^D/u)$ given $R^D>u$ is exponential with mean $\gamma_D$, so that the Hill bias vanishes, $B_D=0$.
The only source of error is then statistical, and consistency means $\widehat\pi_D\to\pi_D$ and $\widehat\gamma_D\to\gamma_D$. \\
\emph{Under (H)}, the tail is only asymptotically Pareto: it deviates from a Pareto tail by a relative error $\eta_D(r)$ that vanishes at the polynomial rate $r^{-\rho_D}$.
This holds for instance for Student radii ($\rho_D=2$) and Fréchet radii ($\rho_D=\alpha_D$).
The Pareto approximation becomes accurate only far in the tail, which requires a threshold $u_n\to\infty$.
The target of $\widehat\pi_D$ then moves with $n$ and vanishes, $S_D(u_n)\to0$, so that $\widehat\pi_D\to\pi_D$ is no longer meaningful; the relevant statement is the relative one, $\widehat\pi_D/S_D(u_n)\to1$, i.e.\ $\delta_D\to0$.

\subsubsection*{Technical lemmas}
The proofs of \autoref{prop:pareto} and \autoref{thm:consistency} both start from a decomposition of $\Delta_n$, established in the proof of \autoref{prop:pareto} (see \autoref{eq:decomp} and \autoref{eq:master}): on $E_n$, for every $r_n\geq u$ and every set $I\subset[u,r_n]$,
\begin{equation*}
    \sup_{r\in I}|\Delta_n(r)| \leq \underbrace{|\widehat\gamma_P-\gamma_P|\,|L_n|
    + \widehat\gamma_P\big(|\delta_P|+|\delta_Q|\big)
    + |\widehat\beta-\beta|\log\frac{r_n}{u}}_{\text{stochastic part: estimation of }(\pi_D,\gamma_D)}
    + \underbrace{\sup_{r\in I}|b_n(r)|}_{\text{deterministic part: Pareto approximation}}.
\end{equation*}
The stochastic part is due to the estimation of $(\pi_D,\gamma_D)$; it also contains the (deterministic) Hill bias $B_D$, through $\widehat\gamma_D-\gamma_D=(\widehat\gamma_D-\mu_D)+B_D$.
The deterministic part is due to the Pareto approximation of the tails at level $u$.
The following lemmas control these terms separately.
\begin{itemize}
    \item \autoref{lem:stoch} controls the sampling fluctuations: the number of exceedances $N_D$ concentrates around $m_D$, and $\delta_D$ and $\widehat\gamma_D-\mu_D$ are of order $m_D^{-1/2}$.
    It only requires moment conditions, and thus applies under both (P) and (H).
    \item \autoref{lem:bias} quantifies the Hill bias $B_D=\mu_D-\gamma_D$, which vanishes under (P) and is of order $u^{-\rho_D}$ under (H), and checks the moment condition $\sup_n\nu_D<\infty$ of
    \autoref{lem:stoch}.
    \item \autoref{lem:param} combines \autoref{lem:stoch} and~\autoref{lem:bias} into rates for the estimated parameters $\widehat\gamma_D$, $\widehat\pi_D$ (through $\delta_D$) and $\widehat\beta$, and shows that $\mathbb P(E_n)\to1$.
    \item \autoref{lem:approx} controls the approximation bias $b_n$ of the Pareto model: it vanishes above $u_\star$ under (P), and is of order $u_n^{-\rho_\star}$ under (H).
\end{itemize}

\begin{lemma}[Stochastic part]\label{lem:stoch}
Fix $D\in\{P,Q\}$.
Assume $m_D\to\infty$ and $\sup_n\nu_D<\infty$ where $m_D$ and $\nu_D$ are defined as in \autoref{eq:auxiliary_quantities}.
Then,
    \begin{enumerate}
        \item[(a)] $N_D/m_D-1=\mathrm O_P(m_D^{-1/2})$ and $\delta_D=\mathrm O_P(m_D^{-1/2})$;
        \item[(b)] $\mathbb P(N_D=0)\leq 1/m_D\to0$;
        \item[(c)] $\widehat\gamma_D-\mu_D=\mathrm O_P(m_D^{-1/2})$.
    \end{enumerate}
\end{lemma}

\begin{proof}
    For the sake of readability, we drop the index $D$. \\
    
    $(a)$ Since $N\sim\mathrm{Bin}(n,S(u))$, we have $\E[N]=nS(u)=m$ and $\mathrm{Var}[N]=m(1-S(u))\leq m$.
    By Bienaymé--Tchebychev, $\mathbb P(|N/m-1|>t\,m^{-1/2})\leq t^{-2}$ for all $t>0$, i.e.\ $N/m-1=\mathrm O_P(m^{-1/2})$.
    Since $\widehat\pi/S(u)=N/m$ and $|\log(1+x)|\leq|x|/(1-|x|)\leq2|x|$ for $|x|\leq1/2$, we get
    $|\delta|\leq 2|N/m-1|$ on $\{|N/m-1|\leq1/2\}$, an event of probability at least $1-4/m\to1$.
    Hence $\delta=\mathrm O_P(m^{-1/2})$.
    
    \noindent $(b)$ It is enough to note that $\mathbb P(N=0)\leq\mathbb P(|N/m-1|\geq1)\leq\mathrm{Var}(N/m)\leq 1/m$.
    
    \noindent $(c)$ Let the random variables $Y_i=\log(R_i/u)\mathbf 1_{\{R_i>u\}}$, so that $\E[Y_i]=S(u)\mu$ and
    $\E[Y_i^2]=S(u)\nu$, and let $\overline Y=m^{-1}\sum_{i=1}^n Y_i$.
    Then $\E[\overline Y]=nS(u)\mu/m=\mu$ and $\mathrm{Var}[\overline Y]\leq n\E[Y_1^2]/m^2=\nu/m$, so $\overline Y-\mu=\mathrm O_P(m^{-1/2})$ again by Bienaymé-Tchebychev.
    On $\{N\geq1\}$, $\widehat\gamma=\overline Y/(N/m)$, hence
    \begin{equation*}
        \widehat\gamma-\mu = \frac{(\overline Y-\mu)-\mu\,(N/m-1)}{N/m}.
    \end{equation*}
    On $\{N/m\geq1/2\}$, whose probability tends to one by (a), $|\widehat\gamma-\mu|\leq2\big(|\overline Y-\mu|+\mu|N/m-1|\big)$, and $\mu\leq\nu^{1/2}$ (Cauchy--Schwarz) is bounded.
    Thus $\widehat\gamma-\mu=\mathrm O_P(m^{-1/2})$.
\end{proof}

\begin{lemma}[Hill bias]\label{lem:bias}
Fix $D\in\{P,Q\}$.
\begin{enumerate}
    \item[(i)] Under (P): $\mu_D=\gamma_D$ and $\nu_D=2\gamma_D^2$, i.e.\ $B_D=0$.
    \item[(ii)] Under (H), for every $u\geq r_0$ such that $Ku^{-\rho_D}\leq1/2$:
        \begin{equation*}
            |B_D|\leq 4K\gamma_D\,u^{-\rho_D}
            \qquad\text{and}\qquad
            \nu_D\leq 6\gamma_D^2 .
        \end{equation*}
    \end{enumerate}
\end{lemma}

\begin{proof}
    Drop the index $D$.
    For the nonnegative variable $Z=\log(R/u)$ given $R>u$,
    \begin{equation}\label{eq:moments}
        \mu=\int_0^\infty\frac{S(ue^s)}{S(u)}\mathrm ds, \qquad \nu=\int_0^\infty 2s\,\frac{S(ue^s)}{S(u)}\mathrm ds .
    \end{equation}
    (i) Under (P), $S(ue^s)/S(u)=e^{-\alpha s}$, so $Z\sim\mathrm{Exp}(\alpha)$, $\mu=1/\alpha$ and $\nu=2/\alpha^2$. \\
    
    \noindent (ii) Under (H), for $u\geq r_0$ with $Ku^{-\rho}\leq1/2$ and all $s\geq0$,
    \begin{equation*}
        \left|\frac{S(ue^s)}{S(u)}-e^{-\alpha s}\right| = e^{-\alpha s}\,\frac{|\eta(ue^s)-\eta(u)|}{1+\eta(u)}
        \leq e^{-\alpha s}\,\frac{2Ku^{-\rho}}{1/2} = 4Ku^{-\rho}e^{-\alpha s},
    \end{equation*}
    using $|\eta(ue^s)|\leq K(ue^s)^{-\rho}\leq Ku^{-\rho}$.
    Integrating over $s\in(0,\infty)$ in \eqref{eq:moments} gives $|\mu-1/\alpha|\leq4Ku^{-\rho}/\alpha$.
    Moreover, the same bound and $4Ku^{-\rho}\leq2$ give $S(ue^s)/S(u)\leq3e^{-\alpha s}$, whence $\nu\leq3\int_0^\infty2se^{-\alpha s}\mathrm ds=6/\alpha^2$.
\end{proof}

\begin{lemma}[Estimated parameters]\label{lem:param}
    Assume that, for $D\in\{P,Q\}$, $m_D\to\infty$, $\sup_n\nu_D<\infty$ and $B_D\to0$.
    Then,
    \begin{enumerate}
        \item[(a)] $\mathbb P(E_n)\to1$;
        \item[(b)] For $D\in\{P,Q\}$,
            \begin{equation}\label{eq:gammaerr}
                \widehat\gamma_D-\gamma_D=\mathrm O_P(m_D^{-1/2})+B_D, \qquad \delta_D=\mathrm O_P(m_D^{-1/2}).
            \end{equation}
        \item[(c)] We have $\widehat\gamma_P=\mathrm O_P(1)$ and
            \begin{equation}\label{eq:betaerr}
                |\widehat\beta-\beta|=\mathrm O_P\big(m_P^{-1/2}+m_Q^{-1/2}\big)+O\big(|B_P|+|B_Q|\big).
            \end{equation}
    \end{enumerate}
\end{lemma}

\begin{proof}
    $(a)$ By \autoref{lem:stoch}(b), $\mathbb P(E_n^c)=\mathbb P(\{N_P=0\}\cup \{N_Q=0\})\leq1/m_P+1/m_Q\to0$.
    Then, $\mathbb P(E_n)\to1$.

    \noindent $(b)$ Throughout, a sequence that is $\mathrm \mathrm O_P(a_n)$ (resp. $\mathrm o_P(1)$) on events of probability tending to one is $\mathrm O_P(a_n)$ (resp. $\mathrm o_P(1)$).
    Writing $\widehat\gamma_D-\gamma_D=(\widehat\gamma_D-\mu_D)+B_D$, \autoref{eq:gammaerr} follows from \autoref{lem:stoch}(a)(c). \\
    
    \noindent $(c)$ In particular, we have $\widehat\gamma_D\overset{P}{\to}\gamma_D>0$.
    By the continuous mapping theorem applied to $g:x\mapsto1/x$ at $x=\gamma_Q>0$, we get $1/\widehat\gamma_Q=\mathrm O_P(1)$; also $\widehat\gamma_P=\mathrm O_P(1)$.
    Finally, since
    \begin{equation}\label{eq:beta}
        \widehat\beta-\beta=\frac{\widehat\gamma_P}{\widehat\gamma_Q}-\frac{\gamma_P}{\gamma_Q}= \frac{\gamma_Q(\widehat\gamma_P-\gamma_P)-\gamma_P(\widehat\gamma_Q-\gamma_Q)}{\gamma_Q\,\widehat\gamma_Q},
    \end{equation}
    $|\widehat\beta-\beta|=\mathrm O_P\big(|\widehat\gamma_P-\gamma_P|+|\widehat\gamma_Q-\gamma_Q|\big)$, hence \eqref{eq:betaerr}.
\end{proof}

In the following lemma, since under the fixed-threshold assumption (P) the targets of the estimators do not depend on $n$, we suppress the subscript $n$ in $T^\circ_n$ and $b_n$.

\begin{lemma}[Approximation bias]\label{lem:approx} The following approximation properties hold.
    \begin{enumerate}
        \item[(i)] Under (P), let $u_\star=u\max\{1,(\pi_Q/\pi_P)^{1/\alpha_Q}\}$.
            Then $T_R(r)=T^\circ(r)$, i.e.\ $b(r)=0$, for every $r>u_\star$.
        \item[(ii)] Under (H) for $P$ and $Q$, let $C=(c_P/c_Q)^{\gamma_P}$. Then, as $r\to\infty$,
            \begin{equation*}
                \log T_R(r)=\log C+\beta\log r+\varepsilon(r),
                \quad\text{with}\quad
                \varepsilon(r)=O\big(r^{-\beta\rho_P}+r^{-\rho_Q}\big).
            \end{equation*}
            and, whenever $u_n\to\infty$, $\sup_{r\geq u_n}|b_n(r)|=O(u_n^{-\rho_\star})$, with $\rho_\star=\min\{\rho_Q,\ \rho_P\min(1,\beta)\}$.
    \end{enumerate}
\end{lemma}

\begin{proof}
    Recall from the setting that $S_D(r)=\mathbb P(R^D>r)$ is the continuous, positive, nonincreasing survival function of the radius under $D\in\{P,Q\}$, that $S_D^{-1}(p)=\inf\{y\geq0:\ S_D(y)\leq p\}$, $p\in(0,1)$, is its generalized inverse, so that $T_R=S_P^{-1}\circ S_Q$, and that $u$ is the common deterministic threshold, fixed under (P) and equal to $u_n\to\infty$ under (H). \\
    
    \noindent $(i)$ Let $p\in(0,\pi_P)$.
    For $y<u$, $S_P(y)\geq S_P(u)=\pi_P>p$; for $y\geq u$, $S_P(y)\leq p$ iff $y\geq u(\pi_P/p)^{\gamma_P}$.
    Hence $S_P^{-1}(p)=u(\pi_P/p)^{\gamma_P}$.
    For $r>u_\star$, taking $p=S_Q(r)=\pi_Q(r/u)^{-\alpha_Q}<\pi_P$, so $T_R(r)=u(\pi_P/\pi_Q)^{\gamma_P}(r/u)^{\alpha_Q\gamma_P}=T^\circ(r)$. \\
    
    \noindent $(ii)$ Recall that, under (H),
    \begin{equation*}
        \eta_D(r)=\frac{S_D(r)}{c_D\,r^{-\alpha_D}}-1\quad\text{and}\quad |\eta_D(r)|\leq K r^{-\rho_D}\quad\text{for } r\geq r_0,
    \end{equation*}
    is the relative error between $S_D$ and the Pareto tail $c_Dr^{-\alpha_D}$.
    In particular $1+\eta_D(r)>0$, since $S_D>0$, and $\eta_D(r)\to0$ as $r\to\infty$. \\
    
    \emph{Expansion of $T_R$.} For $r$ large enough, $S_Q(r)\in(0,1)$, and $y=T_R(r)$ satisfies $S_P(y)=S_Q(r)$ (see the remark following (S1)-(S3)).
    Since $S_Q(r)\to0$ and $S_P$ is positive and nonincreasing, this forces $T_R(r)\to\infty$ as $r\to\infty$.
    Writing (H) on both sides of $S_P(y)=S_Q(r)$ gives $c_Py^{-\alpha_P}(1+\eta_P(y))=c_Qr^{-\alpha_Q}(1+\eta_Q(r))$.
    Taking logarithms and multiplying by $\gamma_P=1/\alpha_P$, using $\gamma_P\alpha_Q=\beta$, we get
    $\log T_R(r)=\log C+\beta\log r+\varepsilon(r)$ with
    \begin{equation*}
        \varepsilon(r) = \gamma_P\Big[\log\big(1+\eta_P(T_R(r))\big)-\log\big(1+\eta_Q(r)\big)\Big].
    \end{equation*}
    As $r\to\infty$, both arguments $T_R(r)$ and $r$ tend to infinity, so $\eta_P(T_R(r))\to0$ and $\eta_Q(r)\to0$ by (H).
    Hence, $\varepsilon(r)\to0$. \\
    
    \emph{Rate.} Since $\varepsilon(r)\to0$, $e^{\varepsilon(r)}\geq1/2$, \ie $T_R(r)=Cr^\beta e^{\varepsilon(r)}\geq Cr^\beta/2$, for $r$ large.
    Take moreover $r$ large enough so that $r\geq r_0$, $T_R(r)\geq r_0$, $KT_R(r)^{-\rho_P}\leq1/2$ and $Kr^{-\rho_Q}\leq1/2$.
    Then $|\eta_P(T_R(r))|,|\eta_Q(r)|\leq1/2$, and $|\log(1+x)|\leq2|x|$ for $|x|\leq1/2$ gives
    \begin{equation*}
        |\varepsilon(r)| \leq2\gamma_PK\big(T_R(r)^{-\rho_P}+r^{-\rho_Q}\big) \leq2\gamma_PK\big((C/2)^{-\rho_P}r^{-\beta\rho_P}+r^{-\rho_Q}\big) = \mathrm O\big(r^{-\beta\rho_P}+r^{-\rho_Q}\big).
    \end{equation*}
    \emph{Approximation bias.} On the other hand, $L_n=\log(c_P/c_Q)+(\alpha_Q-\alpha_P)\log u_n+\ell_n$ with $\ell_n=\log(1+\eta_P(u_n))-\log(1+\eta_Q(u_n))=\mathrm O(u_n^{-\rho_P}+u_n^{-\rho_Q})$, so that, using $\gamma_P\alpha_P=1$,
    \begin{equation*}
        \log (T^\circ_n(r))= \log u_n+\gamma_PL_n+\beta\log\Big(\frac r{u_n}\Big) = \log C+\beta\log r+\gamma_P\ell_n .
    \end{equation*}
    Therefore $b_n(r)=\varepsilon(r)-\gamma_P\ell_n$, and taking the supremum over $r\geq u_n$ gives $\sup_{r\geq u_n}|b_n(r)|=O\big(u_n^{-\beta\rho_P}+u_n^{-\rho_Q}+u_n^{-\rho_P}\big) =\mathrm O(u_n^{-\rho_\star})$.
    Hence the result.
\end{proof}

\subsection{The exact Pareto case}

Before proving \autoref{thm:consistency}, we treat the well-specified case where both radial tails are exactly Pareto above a fixed threshold.
This result is not stated in the main text, where it is only mentioned after \autoref{thm:consistency}: it isolates the purely statistical error, since both bias terms vanish, and yields the parametric rate together with
almost sure and absolute consistency.
Its proof also establishes the error decomposition~\autoref{eq:decomp} on which the proof of the theorem relies.

\begin{proposition}[Exact Pareto tails]\label{prop:pareto}
    Assume (S1)--(S3) and (P) for $P$ and $Q$, and $n_P,n_Q\to\infty$; let $\underline n=\min(n_P,n_Q)$, $u_\star$ as in \autoref{lem:approx}(i), and $(r_n)$ a sequence with $r_n>u_\star$.
    Then
    \begin{equation}\label{eq:rateA}
        \sup_{u_\star< r\leq r_n}|\Delta_n(r)| = O_P\left(\underline n^{-1/2}\Big(1+\log\frac{r_n}{u}\Big)\right).
    \end{equation}
    In particular:
    \begin{enumerate}
        \item[(i)] $\widehat T_R(r)\to T_R(r)$ almost surely for every $r> u_\star$, and $\sup_{r\in[u,M]}|\widehat T_R(r)-T^\circ(r)|\to0$ almost surely for every $M>u$;
        \item[(ii)] \emph{Relative uniform consistency} $\sup_{u_\star< r\leq r_n}\big|\widehat T_R(r)/T_R(r)-1\big|\overset{P}{\longrightarrow}0$ as soon as $\log r_n=o(\underline n^{1/2})$;
        \item[(iii)] \emph{Absolute uniform consistency} $\sup_{u_\star< r\leq r_n}\big|\widehat T_R(r)-T_R(r)\big|\overset{P}{\longrightarrow}0$ as soon as $r_n^{\beta}\,\big(1+\log(r_n/u)\big)=o(\underline n^{1/2})$.
    \end{enumerate}
\end{proposition}

\emph{Error decomposition.} We first decompose $\Delta_n$.
On $E_n$, since $\log(\widehat\pi_P)-\log(\widehat\pi_Q)=L_n+\delta_P-\delta_Q$, we have, for
every $r\geq u$,
\begin{equation*}
    \log(\widehat T_R(r))=\log u+\widehat\gamma_P\,(L_n+\delta_P-\delta_Q)+\widehat\beta\log\frac ru, \quad\text{and}\quad \log( T^\circ_n(r))=\log u+\gamma_P L_n+\beta\log\Big(\frac{r}{u}\big).
\end{equation*}
Using that $\log(T^\circ_n(r))-\log (T_R(r))=-b_n(r)$, we get, on $E_n$ and for every $r\geq u$,
\begin{equation}\label{eq:decomp}
    \Delta_n(r) = (\widehat\gamma_P-\gamma_P)\,L_n
    + \widehat\gamma_P(\delta_P-\delta_Q)
    + (\widehat\beta-\beta)\log\Big(\frac{r}{u}\Big) - b_n(r).
\end{equation}
Consequently, since $0\leq\log(r/u)\leq\log(r_n/u)$ for $r\in[u,r_n]$, for every $r_n\geq u$ and every set $I\subset[u,r_n]$, on $E_n$,
\begin{equation}\label{eq:master}
    \sup_{r\in I}|\Delta_n(r)| \leq |\widehat\gamma_P-\gamma_P|\,|L_n|
    + \widehat\gamma_P\big(|\delta_P|+|\delta_Q|\big)
    + |\widehat\beta-\beta|\log\frac{r_n}{u} + \sup_{r\in I}|b_n(r)|.
\end{equation}

\emph{Rate \autoref{eq:rateA}.} By \autoref{lem:bias}(i), $B_D=0$ and $\nu_D=2\gamma_D^2$, while
$m_D=n_D\pi_D\geq\pi_D\,\underline n\to\infty$.
\autoref{lem:param} thus applies and gives $\widehat\gamma_D-\gamma_D$, $\delta_D$ and $\widehat\beta-\beta$ all of order $O_P(\underline n^{-1/2})$, with $\widehat\gamma_P=O_P(1)$.
Moreover $L_n=\log(\pi_P/\pi_Q)$ is constant and $b(r)=0$ for $r>u_\star$ (\autoref{lem:approx}(i)).
Plugging these estimates into \autoref{eq:master} with $I=(u_\star,r_n]$ gives \autoref{eq:rateA}.

(i) Let $Y_i^D:=\log(R_i^D/u)\mathbf 1_{\{R_i^D>u\}}$.
By the strong law of large numbers, $\widehat\pi_D\to\pi_D$ and $n_D^{-1}\sum_iY^D_i\to\pi_D\gamma_D$ almost surely, so $\widehat\gamma_D\to\gamma_D$ almost surely by the continuous mapping theorem ($(x,y)\mapsto x/y$ is continuous at $(\pi_D\gamma_D,\pi_D)$).
Set $\theta=(\gamma_P,\gamma_Q,\pi_P,\pi_Q)\in(0,\infty)^4$ (resp. $\widehat\theta=(\widehat\gamma_P,\widehat\gamma_Q,\widehat\pi_P,\widehat\pi_Q)$) and $\Phi(\theta;r)=u(\pi_P/\pi_Q)^{\gamma_P}(r/u)^{\gamma_P/\gamma_Q}$, so that we have $T^\circ=\Phi(\theta;\cdot)$ (resp. $\widehat T_R=\Phi(\widehat\theta;\cdot)$).
Since $\Phi$ is jointly continuous on $(0,\infty)^4\times[u,M]$ and $[u,M]$ is compact, the map $\theta\mapsto\Phi(\theta;\cdot)$ is continuous from $(0,\infty)^4$ into $\big(C([u,M]),\|\cdot\|_\infty\big)$.
The continuous mapping theorem then gives $\|\widehat T_R-T^\circ\|_{\infty,[u,M]}\to0$ almost surely; pointwise convergence to $T_R(r)$ for $r>u_\star$ follows from $T^\circ=T_R$ there.

(ii) If $\log r_n=\mathrm o(\underline n^{1/2})$, the right-hand side of \autoref{eq:rateA} is $\mathrm o_P(1)$.
Then $\sup|\widehat T_R/T_R-1|=\sup|e^{\Delta_n}-1|\leq Z_ne^{Z_n}$ with $Z_n=\sup|\Delta_n|\overset{P}{\to}0$.

(iii) Since $T_R$ is nondecreasing, $\sup_{u_\star< r\leq r_n}|\widehat T_R(r)-T_R(r)| =\sup T_R|e^{\Delta_n}-1|\leq T_R(r_n)\,Z_ne^{Z_n}$, where $T_R(r_n)=T^\circ(r_n)=\mathrm O(r_n^\beta)$ and $Z_n=\mathrm O_P\big(\underline n^{-1/2}(1+\log(r_n/u))\big)$ by \autoref{eq:rateA}.
The product is $\mathrm o_P(1)$ under the stated condition, since $e^{Z_n}=\mathrm O_P(1)$.

\subsection{Proof of \autoref{thm:consistency}} Since $u_n\to\infty$, we may assume $u_n\geq\max(r_0,1)$ and
$Ku_n^{-\rho}\leq1/2$. \autoref{lem:bias}(ii) then gives $|B_D|=O(u_n^{-\rho})$ and $\nu_D\leq6\gamma_D^2$, so \autoref{lem:param} applies: $|\widehat\gamma_P-\gamma_P|+|\widehat\beta-\beta| =\mathrm O_P(m_P^{-1/2}+m_Q^{-1/2}+u_n^{-\rho})$ and $\widehat\gamma_P(|\delta_P|+|\delta_Q|)=\mathrm O_P(m_P^{-1/2}+m_Q^{-1/2})$.
By the proof of \autoref{lem:approx}(ii), $|L_n|\leq|\log(c_P/c_Q)|+|\alpha_Q-\alpha_P|\log u_n+1=O(1+\log r_n)$ for $n$ large, and $\log(r_n/u_n)\leq\log r_n$.
Finally, $\sup_{r\geq u_n}|b_n(r)|=O(u_n^{-\rho_\star})$ by \autoref{lem:approx}(ii).
The decomposition \autoref{eq:decomp}, and hence the bound \autoref{eq:master}, established in the proof of
\autoref{prop:pareto} without using (P), hold here with $u=u_n$.
Plugging the above estimates into \autoref{eq:master} with $I=[u_n,r_n]$ gives \autoref{eq:rate}.
The relative consistency follows as in \autoref{prop:pareto}(ii), noting that $u_n^{-\rho_\star}\to0$.

\section{Experimental Details}
\label{sec:supp-exp-details}

\subsection{Datasets and preprocessing}

The isotropic synthetic target is a symmetric \(150\)-dimensional \(\alpha\)-stable law with \(\alpha=1.7\).
The second synthetic target is a \(150\)-dimensional, \(16\)-component stable mixture with \(\alpha=1.7\), rank-\(6\) center structure, imbalance parameter \(1.2\), mean scale \(7.5\), base scale \(0.55\), and fixed structure seed \(0\).
The HRRR datasets contain hourly precipitation in millimeters at \(150\) fixed eastern or western US locations~\citep{dowell2022hrrr}.
The wildfire dataset contains positive final fire sizes in acres from the USDA FPA-FOD records for 1992--2024~\citep{short2026fpa}.
Tabular models standardize each feature using training-split statistics; generated samples are transformed back to their original units before correction and evaluation.
Sample and evaluation counts are given in \autoref{tab:experiment-settings}.

\subsection{Reference models and sampling baselines}

For each dataset, DDPM, DLPM, and \(t\)-EDM are trained independently using \(25{,}000\) training samples and \(1{,}000\) validation samples.
DDPM~\citep{ho2020denoising} uses \(v\)-prediction and \(128\) sampling steps. Writing \(x_t=a_t x_0+b_t\epsilon\), with \(\epsilon\sim\mathcal N(0,I)\), its prediction target is \(v_t=a_t\epsilon-b_t x_0\).
DDIM~\citep{song2021denoising} reuses the trained DDPM checkpoint with deterministic sampling (\(\eta=0\)) and is therefore not trained separately.
DLPM~\citep{shariatian2025dlpm} uses \(\epsilon\)-prediction of the injected \(\alpha\)-stable perturbation with stability index \(\alpha=1.8\) and \(128\) sampling steps.
Finally, \(t\)-EDM~\citep{pandey2025heavytailed} uses Student-\(t\) perturbations with \(\nu=2.2\) degrees of freedom and a \(64\)-step stochastic Heun sampler.

Synthetic checkpoints use a time-conditioned MLP with width \(240\), depth \(4\), and time-embedding dimension \(128\); tabular checkpoints use width \(128\), depth \(3\), and time-embedding dimension \(64\).
Both use normalized network blocks and no dropout.
Models are trained using AdamW with weight decay \(10^{-6}\), cosine learning-rate decay, gradient clipping at \(10\), and early stopping with patience \(100\) for at most \(640\) epochs.
Training batch sizes are \(1{,}024\) for synthetic data and \(512\) for tabular data.

Hyperparameters are selected in a \(96\)-epoch pilot using mean validation loss over three trials for each synthetic dataset and two for each real dataset.
The learning rate is selected from \(\{10^{-4},2\cdot10^{-4},5\cdot10^{-4}\}\) for synthetic data and \(\{2\cdot10^{-4},5\cdot10^{-4}\}\) for real data.
For DDPM, the noise scale \(\sigma_{\max}\) is selected from \(\{2,5,8\}\) and \(\{2,5\}\), respectively.
A fresh model is then trained with the selected hyperparameters, retaining the checkpoint with the best validation loss.

\begin{table}[t]
  \centering
  \small
  \begin{tabular}{@{}lll@{}}
    \toprule
    Method & Predictor / noise & Sampling configuration \\
    \midrule
    DDPM & \(v\)-prediction & Stochastic DDPM, \(128\) steps; pilot-selected \(\sigma_{\max}\) \\
    DDIM & Same DDPM \(v\)-prediction weights & Deterministic DDIM, \(\eta=0\) \\
    DLPM & \(\epsilon\)-prediction & Native sampler, \(128\) steps, \(\alpha=1.8\) \\
    \(t\)-EDM & Student-\(t\) perturbations & Stochastic Heun, \(64\) steps, \(\nu=2.2\) \\
    Zero controls & Zeroed DLPM/\(t\)-EDM networks & Corresponding sampler settings unchanged \\
    RTR & Post-hoc map on DDPM/DDIM endpoints & Same endpoint samples followed by radial rescaling \\
    Test--test & Two disjoint reference halves & Each half uses half the evaluation sample \\
    \bottomrule
  \end{tabular}
  \caption{Sampling methods used in the benchmark. The \(t\)-EDM schedule additionally uses \(\sigma_{\min}=0.005\), \(\sigma_{\max}=5\), and \(\sigma_{\mathrm{data}}=1\); \(\eta\) denotes the DDIM stochasticity parameter.}
  \label{tab:additional-methods}
\end{table}

The zero controls retain the corresponding noise schedule and sampler while setting all network parameters to zero, thereby isolating tail behavior induced by the prior and sampler alone.
Trained and zero-network \(t\)-EDM endpoints use matched initial latents.
The test--test baseline compares two disjoint halves of the reference evaluation sample.
All trained checkpoints are frozen before RTR is fitted; no retraining or fine-tuning of the underlying generator is performed.

\subsection{Radial Tail Repair fitting}

For each trial, RTR uses \(5{,}000\) reference samples and \(5{,}000\) generated samples for each corrected model.
Synthetic reference samples are freshly drawn, while tabular reference samples are drawn without replacement from the training split.
The radial--angular independence scan evaluates \(30\) candidate tail fractions from \(\max(0.01,N_{\min}/N)\) to \(0.30\), where \(N=5{,}000\) is the fitting-sample size and \(N_{\min}=300\) is the minimum number of exceedances.

At each candidate threshold, we use the mean of \(20\) permutation-test \(p\)-values as the radial--angular independence diagnostic.
Each test uses \(100\) random permutations and distance covariance between log-radius distances and unit-direction distances.
For a candidate tail fraction \(f\), the diagnostic uses a subsample of size \(\min\{n_{\mathrm{tail}},\max(20,\mathrm{round}(500f))\}\), where \(n_{\mathrm{tail}}\) is the number of observations above the candidate threshold.
A cutoff of \(0.4\) selects the onset.

A Hill estimate and the empirical exceedance fraction define the fitted Pareto survival above the onset. RTR uses the larger of the reference and generated onsets and a smooth gate with \(\kappa=1.25\), applying only positive radial expansion when the fitted reference tail is heavier than the generated tail. For the one-dimensional wildfire data, radial--angular independence is not identifiable, so this experiment is treated as a scalar-tail application.

\subsection{Evaluation metrics and execution}

Evaluation uses \(21\) logarithmically spaced tail probabilities from \(10^{-1}\) to the dataset-specific \(p_{\min}\) in \autoref{tab:experiment-settings}.
Coordinate thresholds are computed from the reference sample, using absolute coordinate values for synthetic data and the original values for tabular data.
We additionally evaluate radial TCE using \(\|x\|_2\), maximum-event TCE using \(\max_j|x_j|\) for synthetic data and \(\max_j x_j\) for tabular data, and the signed maximum-event log-mass ratio. TCE areas are computed by trapezoidal integration over \(\operatorname{logit}(1-p)\).
Full metric definitions are given in~\autoref{sec:supp-metrics}.

Biased RBF-MMD\(^{2}\) is computed using at most \(1{,}000\) reference and generated samples (\(512\) for HRRR).
Its bandwidth is fixed per dataset from reference pairs only, with \(\gamma=1/(2m_2)\), where \(m_2\) is the median positive squared pairwise distance.
Reference thresholds and the MMD bandwidth are computed once per trial and shared across methods.
Synthetic evaluation uses fresh reference draws, while real-data evaluation uses held-out samples.

Endpoint samples are generated in batches of \(4{,}096\); feature-processing block sizes are \(4\) for synthetic data, \(10\) for HRRR, and \(1\) for wildfire.
The final benchmark uses one trial with seed \(0\), because each synthetic dataset has a \(5\)-million-sample evaluation set and repeating the full evaluation would be computationally costly.
Endpoint generation and evaluation are submitted as separate per-dataset SLURM jobs, each using one NVIDIA V100 GPU, one task, eight CPUs, and a two-hour wall-time limit.

\section{Additional Experimental Results}
\label{sec:supp-exp-additional}

We provide additional tail diagnostics and metric summaries complementing the main results.
The HRRR-west coordinate TCE curve is reported in \autoref{fig:hrrr-west-tce}, while \autoref{fig:max-event-curves} and \autoref{fig:max-event-ratios} report maximum-event TCE and signed tail-mass ratios across all datasets.
Aggregate paired changes are given in \autoref{tab:delta-synthetic} and \autoref{tab:delta-real}, with the corresponding absolute metric values in \autoref{tab:values-synthetic} and \autoref{tab:values-real}.
For paired changes, we report corrected minus uncorrected values, so negative values indicate improvement.

\begin{figure}[t]
  \centering
  \includegraphics[width=0.3\linewidth]{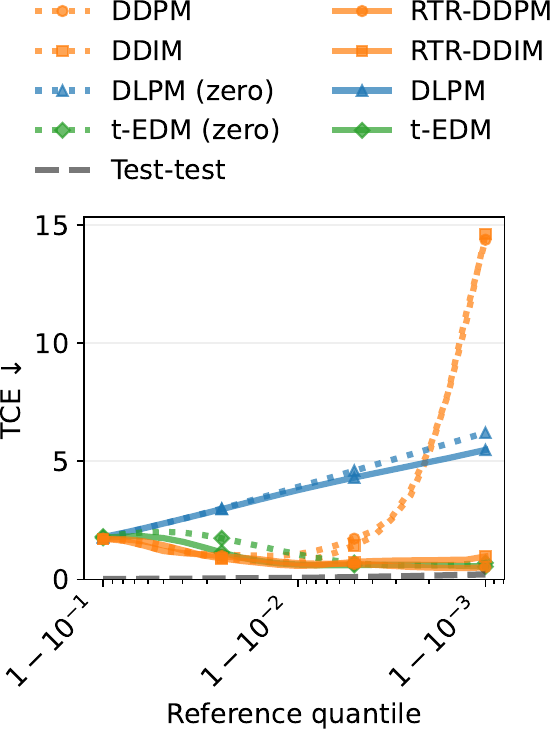}
  \caption{Coordinate TCE for HRRR west.}
  \label{fig:hrrr-west-tce}
\end{figure}

As shown in \autoref{fig:hrrr-west-tce}, HRRR west exhibits the same qualitative behavior as HRRR east: the light-tailed baselines increasingly underrepresent the tail, while RTR substantially reduces the discrepancy.

\begin{figure}[p]
  \centering
  \captionsetup[subfigure]{labelformat=empty,justification=centering}
  \begin{subfigure}[t]{0.3\linewidth}\centering
    \includegraphics[width=\linewidth]{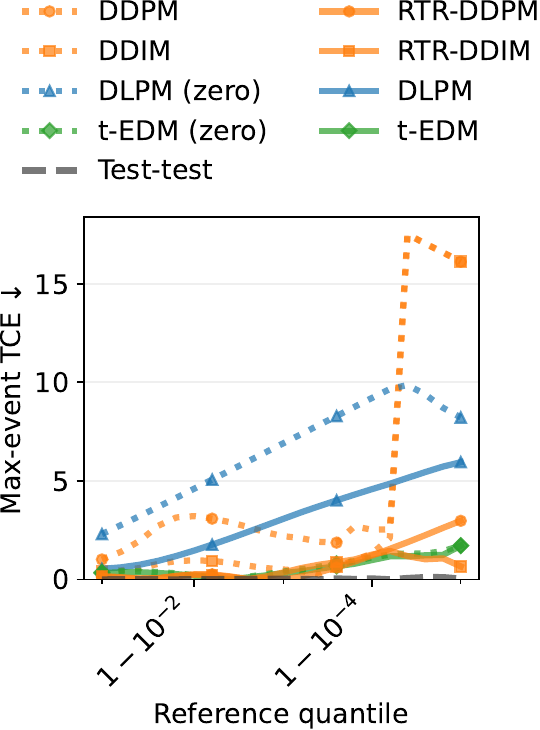}
    \caption{Isotropic stable.}
  \end{subfigure}\hspace{0.03\linewidth}
  \begin{subfigure}[t]{0.3\linewidth}\centering
    \includegraphics[width=\linewidth]{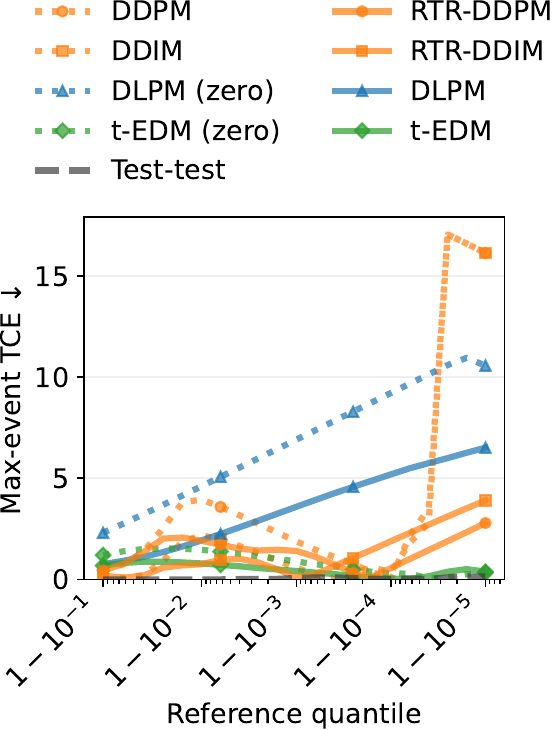}
    \caption{Stable mixture.}
  \end{subfigure}\par\medskip
  \begin{subfigure}[t]{0.3\linewidth}\centering
    \includegraphics[width=\linewidth]{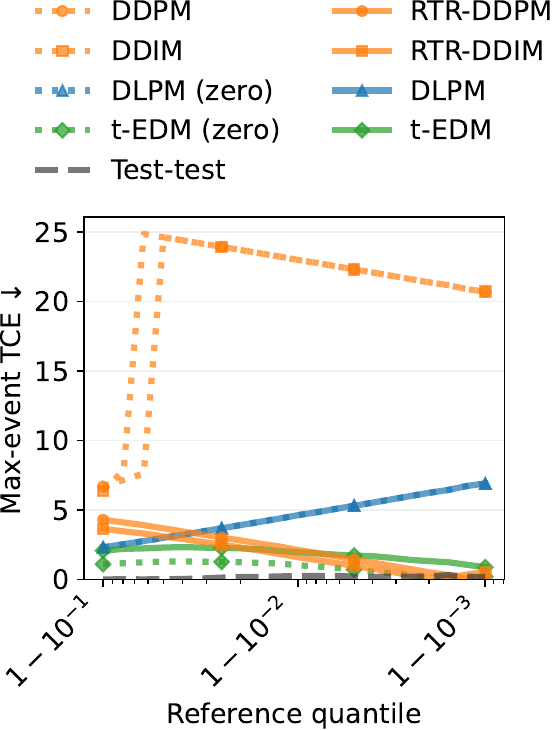}
    \caption{HRRR east.}
  \end{subfigure}\hfill
  \begin{subfigure}[t]{0.3\linewidth}\centering
    \includegraphics[width=\linewidth]{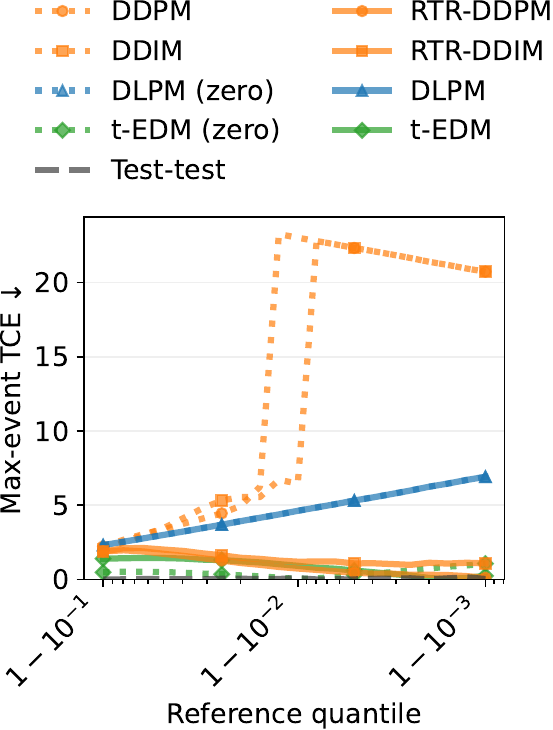}
    \caption{HRRR west.}
  \end{subfigure}\hfill
  \begin{subfigure}[t]{0.3\linewidth}\centering
    \includegraphics[width=\linewidth]{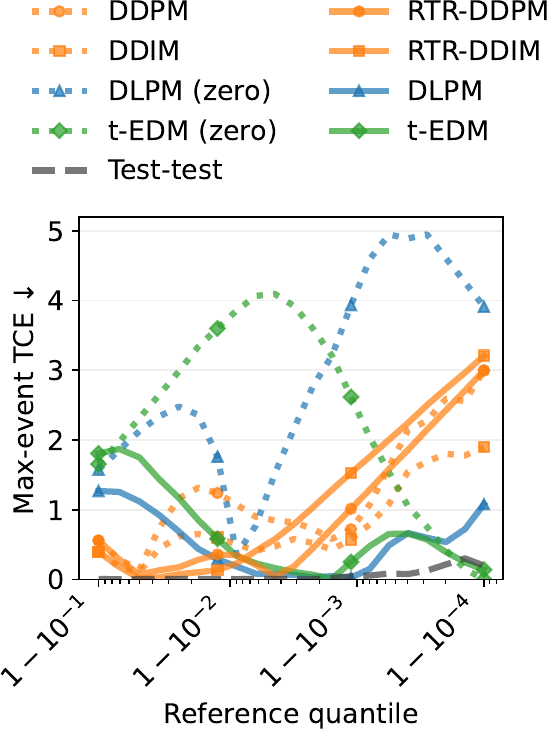}
    \caption{Wildfire.}
  \end{subfigure}
  \caption{Maximum-event TCE across the five datasets. These curves complement coordinate TCE by measuring coverage of joint extreme events.}
  \label{fig:max-event-curves}
\end{figure}

\begin{figure}[p]
  \centering
  \captionsetup[subfigure]{labelformat=empty,justification=centering}
  \begin{subfigure}[t]{0.3\linewidth}\centering
    \includegraphics[width=\linewidth]{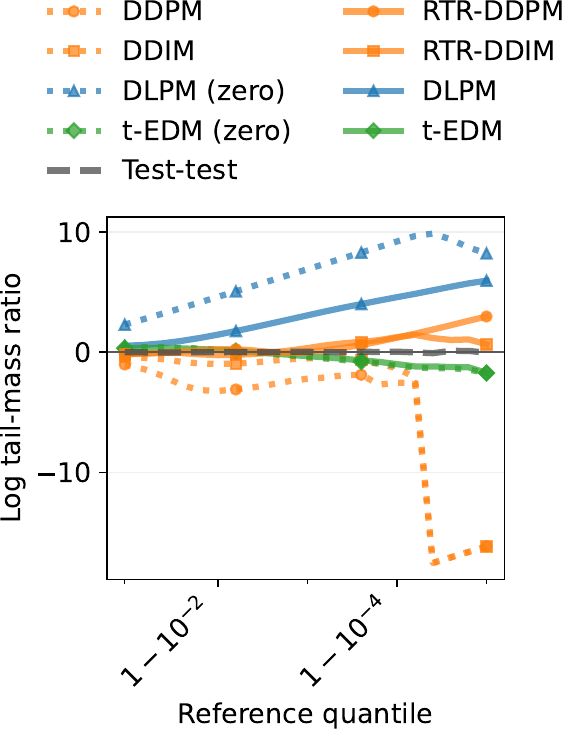}
    \caption{Isotropic stable.}
  \end{subfigure}\hspace{0.03\linewidth}
  \begin{subfigure}[t]{0.3\linewidth}\centering
    \includegraphics[width=\linewidth]{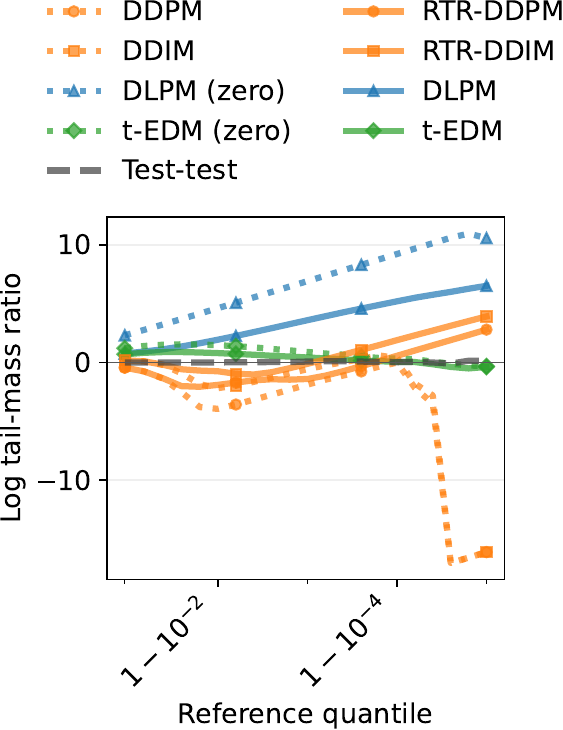}
    \caption{Stable mixture.}
  \end{subfigure}\par\medskip
  \begin{subfigure}[t]{0.3\linewidth}\centering
    \includegraphics[width=\linewidth]{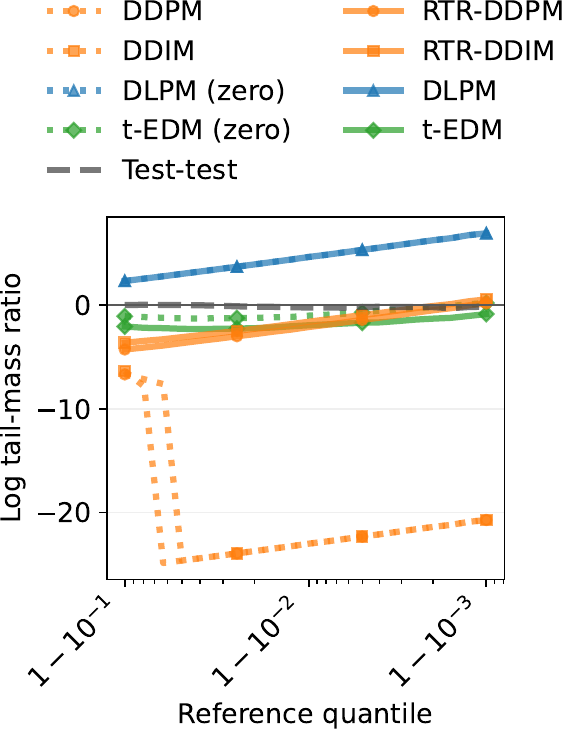}
    \caption{HRRR east.}
  \end{subfigure}\hfill
  \begin{subfigure}[t]{0.3\linewidth}\centering
    \includegraphics[width=\linewidth]{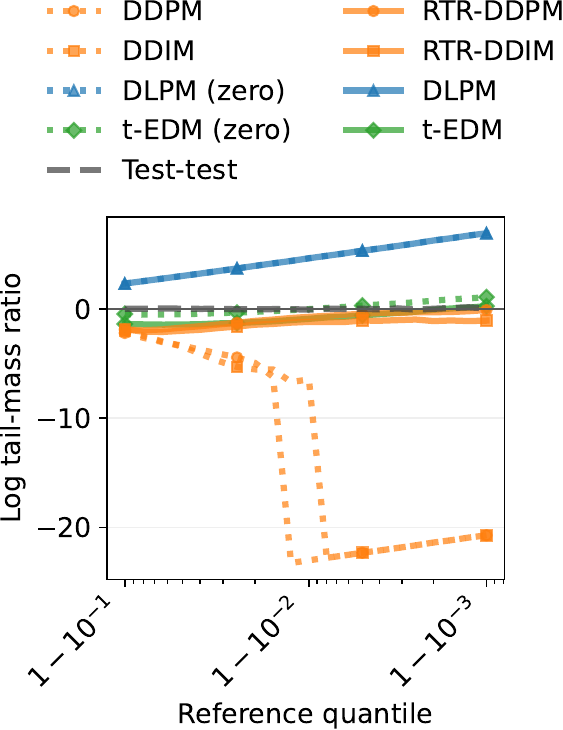}
    \caption{HRRR west.}
  \end{subfigure}\hfill
  \begin{subfigure}[t]{0.3\linewidth}\centering
    \includegraphics[width=\linewidth]{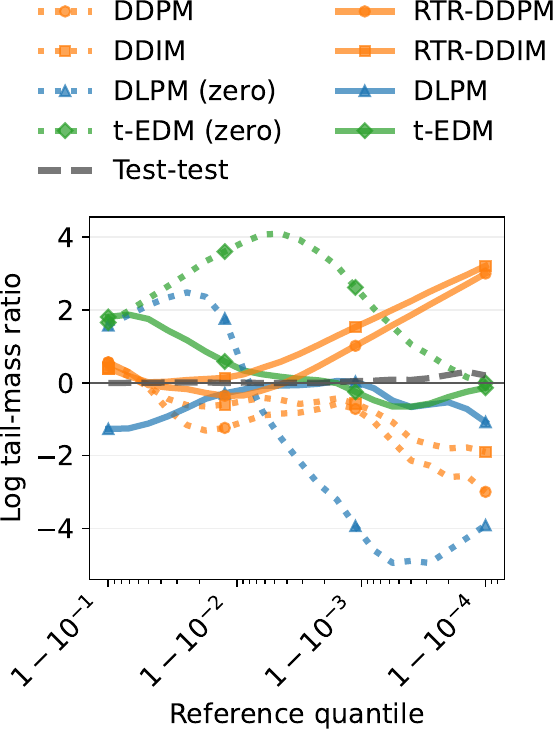}
    \caption{Wildfire.}
  \end{subfigure}
  \caption{Signed maximum-event log tail-mass ratio across the five datasets. Zero indicates calibrated tail mass, while negative and positive values indicate under- and overcoverage, respectively.}
  \label{fig:max-event-ratios}
\end{figure}

The maximum-event results in \autoref{fig:max-event-curves} follow the same broad pattern as coordinate TCE, showing that the improvement extends to joint extreme events.
The signed ratios in \autoref{fig:max-event-ratios} additionally distinguish undercoverage from overcoverage, which absolute TCE cannot capture.
In particular, they reveal cases where RTR crosses from missing tail mass to overcorrection, most visibly for wildfire.

\newcommand{\resultdeltabarplot}[1]{\includegraphics[width=0.15\linewidth]{figures/#1}}
\newcommand{\resultvaluebarplot}[1]{\includegraphics[width=0.25\linewidth]{figures/#1}}

\begin{table}[p]
  \centering
  \small
  \resizebox{0.65\linewidth}{!}{%
    \begin{tabular}{@{}lcccc@{}}
      \toprule
      Dataset & \shortstack{Coordinate\\TCE area} & \shortstack{Radial\\TCE area} & \shortstack{Max-event\\TCE area} & MMD$^2$ \\
      \midrule
      Isotropic stable & \resultdeltabarplot{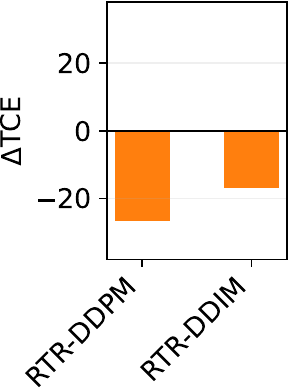} & \resultdeltabarplot{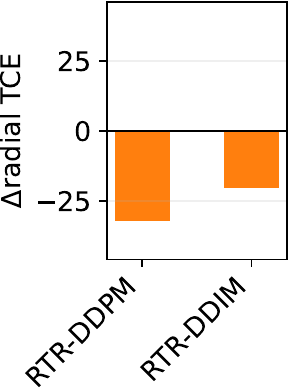} & \resultdeltabarplot{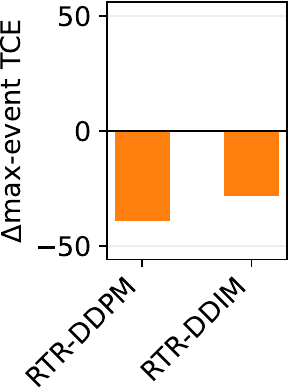} & \resultdeltabarplot{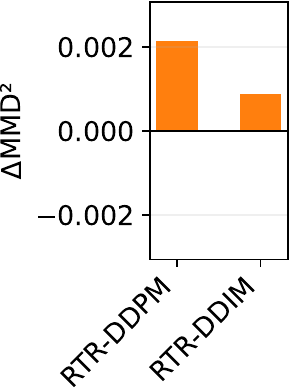} \\
      Stable mixture & \resultdeltabarplot{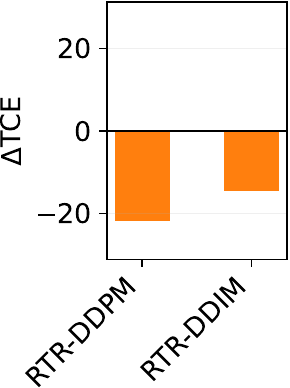} & \resultdeltabarplot{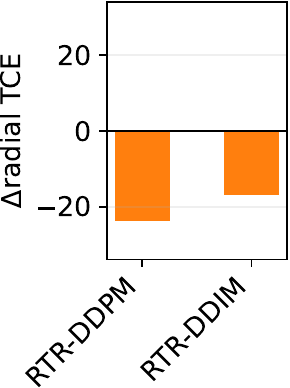} & \resultdeltabarplot{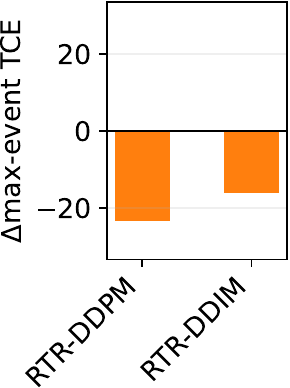} & \resultdeltabarplot{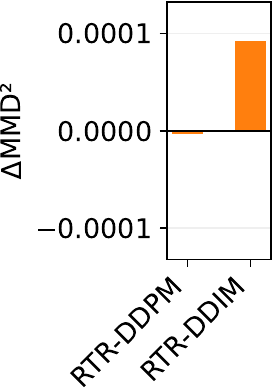} \\
      \bottomrule
    \end{tabular}%
  }
  \caption{Paired metric changes on the synthetic datasets, computed as RTR minus the corresponding uncorrected DDPM or DDIM value. Negative values indicate improvement.}
  \label{tab:delta-synthetic}
\end{table}

\begin{table}[p]
  \centering
  \small
  \resizebox{0.65\linewidth}{!}{%
    \begin{tabular}{@{}lcccc@{}}
      \toprule
      Dataset & \shortstack{Coordinate\\TCE area} & \shortstack{Radial\\TCE area} & \shortstack{Max-event\\TCE area} & MMD$^2$ \\
      \midrule
      HRRR east & \resultdeltabarplot{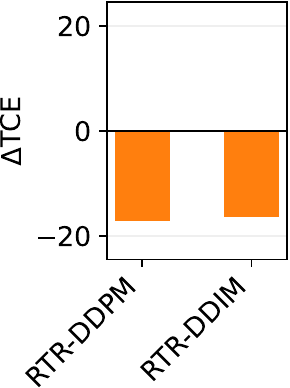} & \resultdeltabarplot{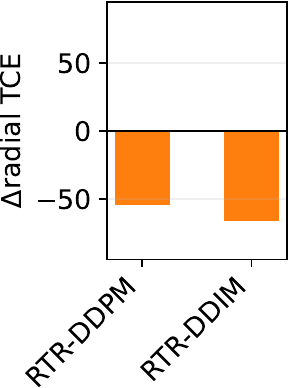} & \resultdeltabarplot{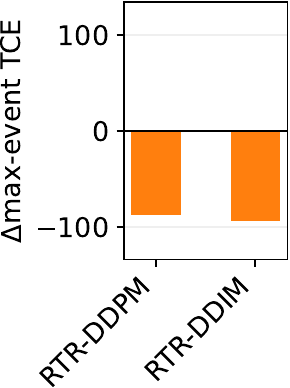} & \resultdeltabarplot{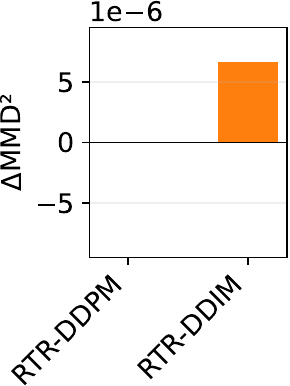} \\
      HRRR west & \resultdeltabarplot{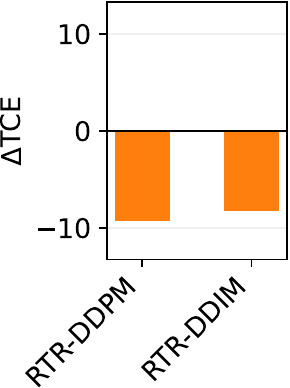} & \resultdeltabarplot{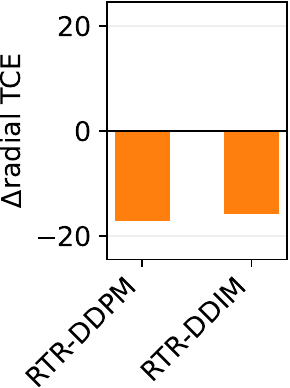} & \resultdeltabarplot{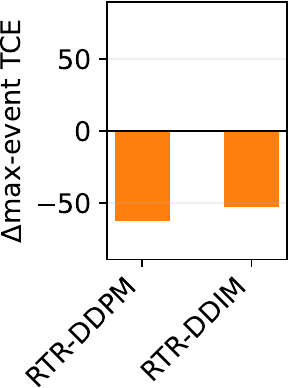} & \resultdeltabarplot{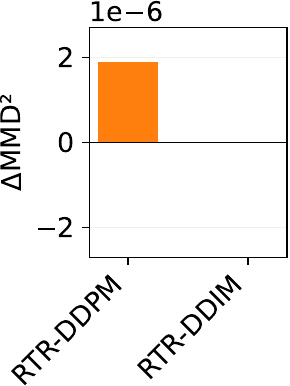} \\
      Wildfire & \resultdeltabarplot{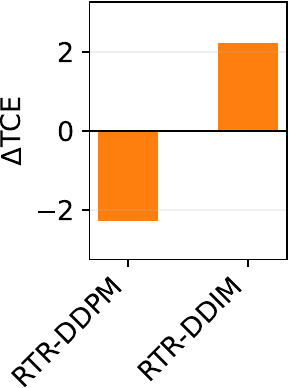} & \resultdeltabarplot{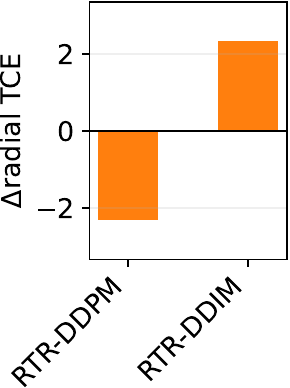} & \resultdeltabarplot{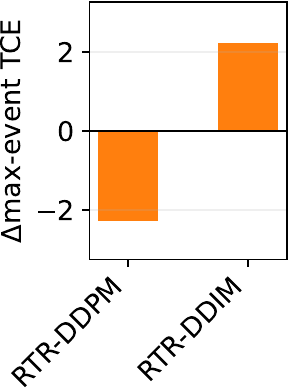} & \resultdeltabarplot{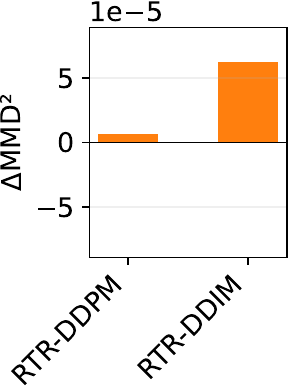} \\
      \bottomrule
    \end{tabular}%
  }
  \caption{Paired metric changes on the real datasets, computed as RTR minus the corresponding uncorrected DDPM or DDIM value. Negative values indicate improvement.}
  \label{tab:delta-real}
\end{table}

The paired changes in \autoref{tab:delta-synthetic} and \autoref{tab:delta-real} quantify the effect of RTR on exactly the same generated endpoints as the corresponding uncorrected model.

\begin{table}[p]
  \centering
  \small
  \resizebox{0.9\linewidth}{!}{%
    \begin{tabular}{@{}lcccc@{}}
      \toprule
      Dataset & \shortstack{Coordinate\\TCE area} & \shortstack{Radial\\TCE area} & \shortstack{Max-event\\TCE area} & MMD$^2$ \\
      \midrule
      Isotropic stable & \shortstack{\autoref{fig:iso-area}} & \resultvaluebarplot{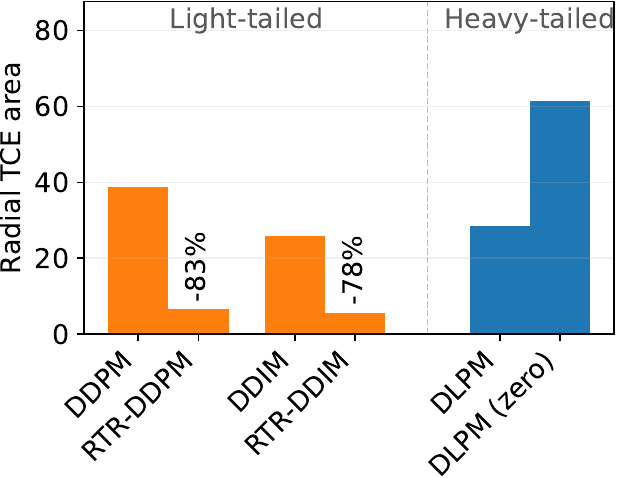} & \resultvaluebarplot{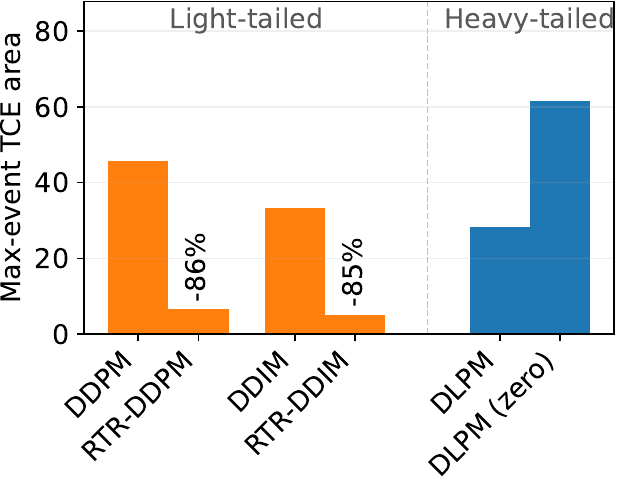} & \shortstack{\autoref{fig:iso-mmd}} \\
      Stable mixture & \shortstack{\autoref{fig:mixture-area}} & \resultvaluebarplot{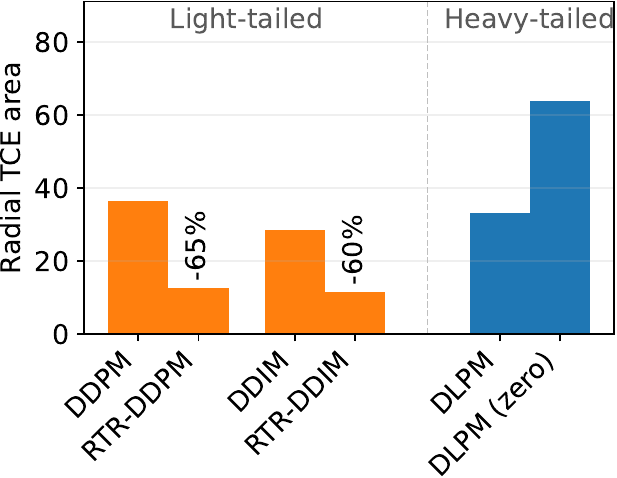} & \resultvaluebarplot{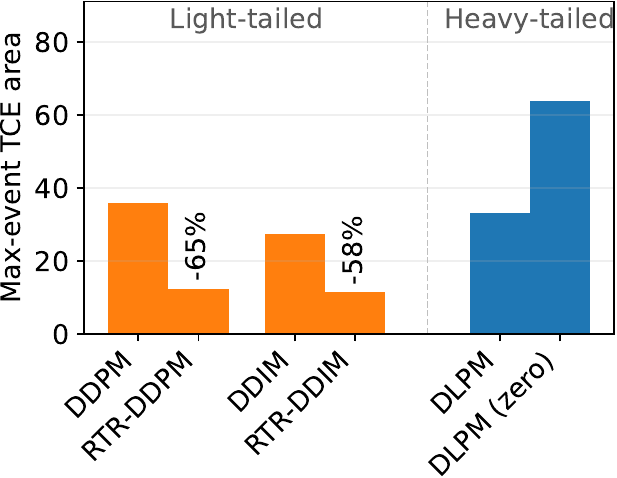} & \shortstack{\autoref{fig:mixture-mmd}} \\
      \bottomrule
    \end{tabular}%
  }
  \caption{Absolute metric values on the synthetic datasets. Coordinate TCE area and MMD$^2$ plots shown in the main text are referenced directly.}
  \label{tab:values-synthetic}
\end{table}

\begin{table}[p]
  \centering
  \small
  \resizebox{0.9\linewidth}{!}{%
    \begin{tabular}{@{}lcccc@{}}
      \toprule
      Dataset & \shortstack{Coordinate\\TCE area} & \shortstack{Radial\\TCE area} & \shortstack{Max-event\\TCE area} & MMD$^2$ \\
      \midrule
      HRRR east & \resultvaluebarplot{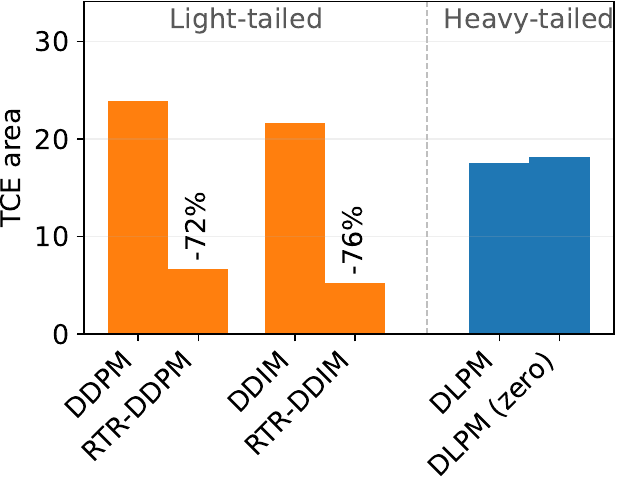} & \resultvaluebarplot{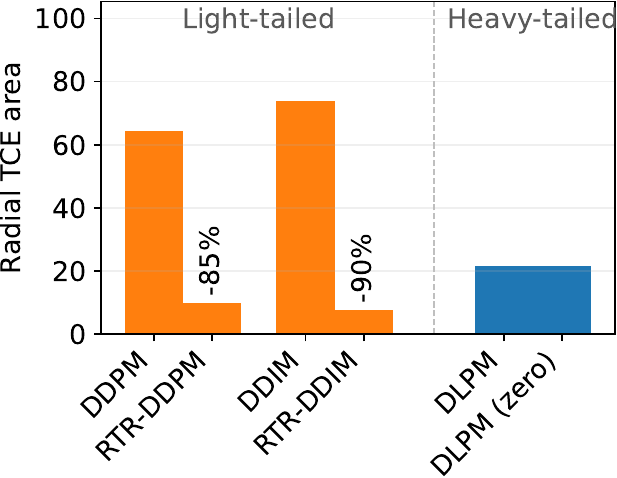} & \resultvaluebarplot{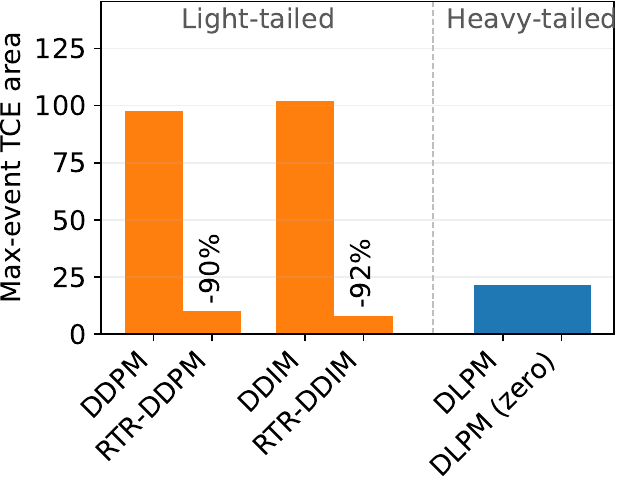} & \resultvaluebarplot{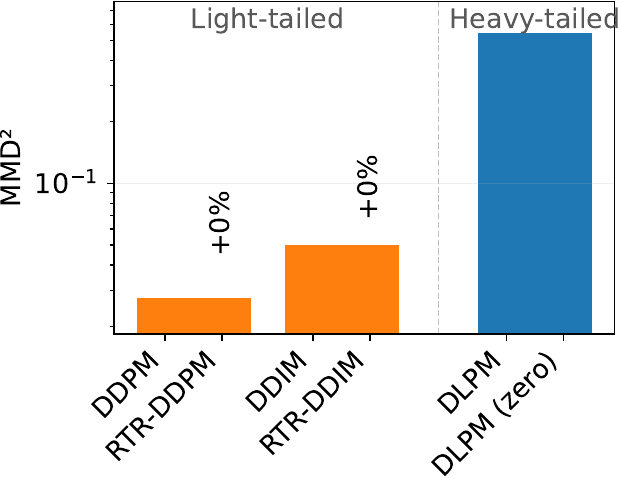} \\
      HRRR west & \resultvaluebarplot{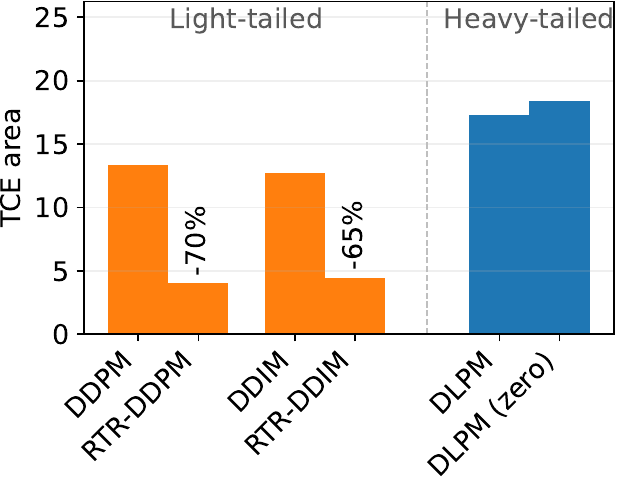} & \resultvaluebarplot{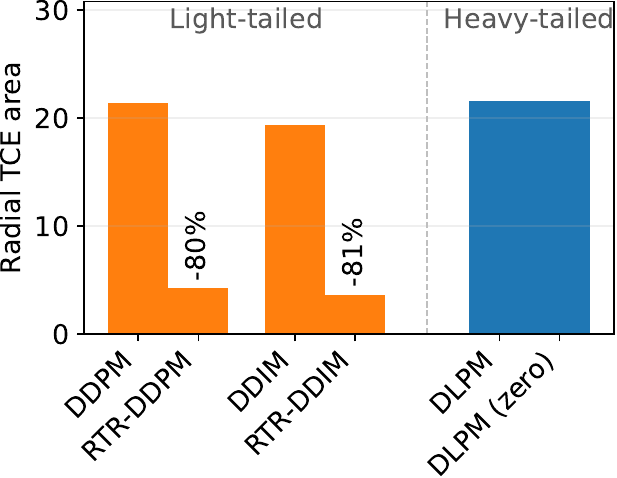} & \resultvaluebarplot{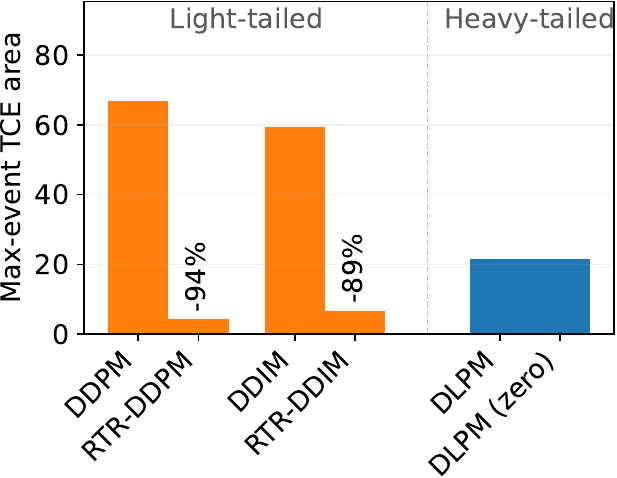} & \resultvaluebarplot{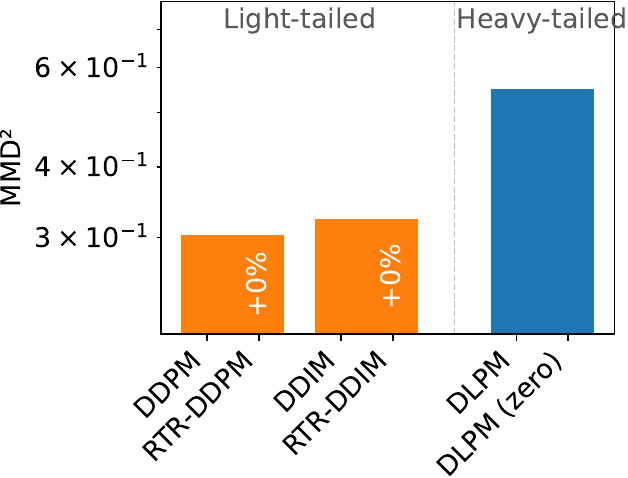} \\
      Wildfire & \resultvaluebarplot{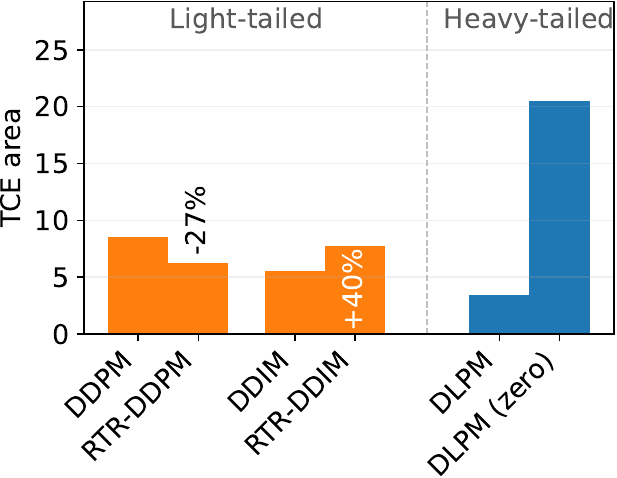} & \resultvaluebarplot{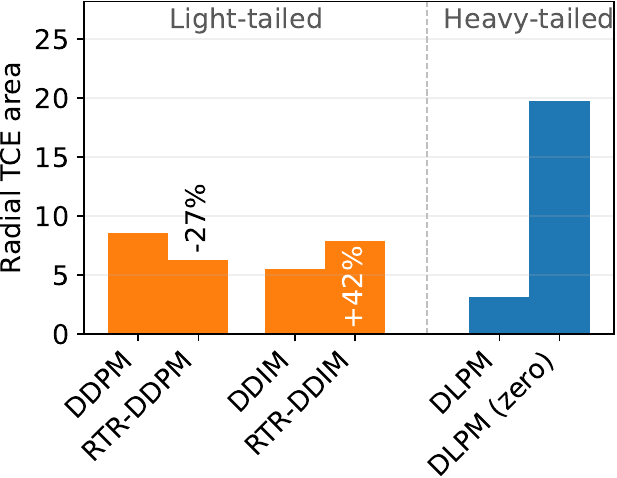} & \resultvaluebarplot{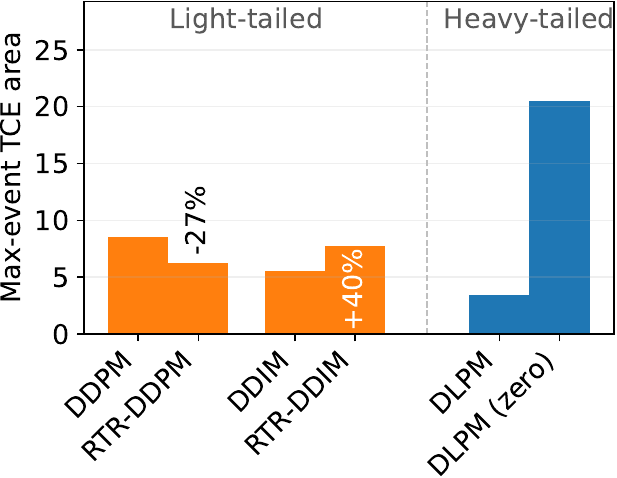} & \resultvaluebarplot{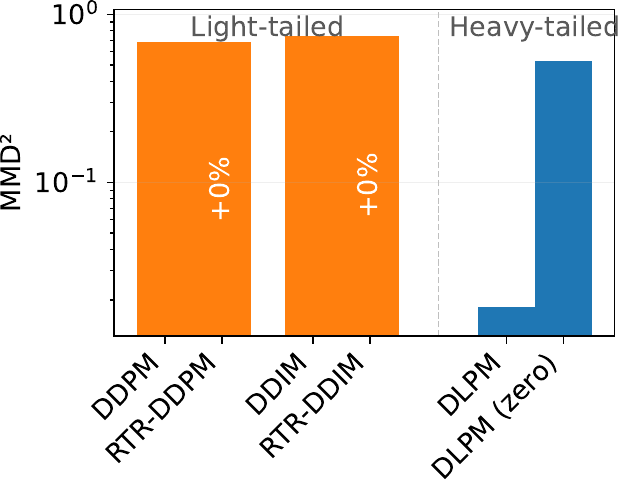} \\
      \bottomrule
    \end{tabular}%
  }
  \caption{Absolute metric values on the real datasets.}
  \label{tab:values-real}
\end{table}

The absolute values in \autoref{tab:values-synthetic} and \autoref{tab:values-real} provide the scale of the corresponding changes.
Across the high-dimensional synthetic and HRRR benchmarks, \autoref{tab:delta-synthetic} and \autoref{tab:delta-real} show consistent reductions in coordinate, radial, and maximum-event TCE, generally with much smaller changes in MMD$^2$.
The wildfire results are less uniform and expose the finite-sample and Pareto-approximation effects discussed in \autoref{sec:discussion}.
\section{Evaluation Metrics}
\label{sec:supp-metrics}

Let \(X=\{x_i\}_{i=1}^{N_P}\) and \(Y=\{y_i\}_{i=1}^{N_Q}\) denote the reference and generated evaluation samples, as in the main text.
Tail coverage is evaluated at \(21\) log-spaced probabilities \(p_1>\cdots>p_{21}\), from \(10^{-1}\) to the dataset-specific \(p_{\min}\) in Table~\ref{tab:experiment-settings}.
For any scalar statistic \(F\), let
\begin{align}
    t_F(p) &= \widehat{\operatorname{Quantile}}_{1-p}\{F(x_i)\}_{i=1}^{N_P},\\
    \widehat S_{D,F}(p) &= \frac{1}{|D|}\sum_{z\in D}\mathbf 1\{F(z)>t_F(p)\},
    \qquad D\in\{X,Y\},\\
    L_F(p) &= \log\frac{\widehat S_{Y,F}(p)+\varepsilon}
    {\widehat S_{X,F}(p)+\varepsilon},
    \qquad \varepsilon=10^{-12}.
\end{align}
Thus, \(L_F(p)=0\) indicates calibrated tail mass, while negative and positive values indicate under- and overcoverage, respectively.

\paragraph{Coordinate, radial, and maximum-event TCE.}
For coordinate \(j\), we use
\[
F_j(x)=
\begin{cases}
|x_j|, & \text{synthetic data},\\
x_j, & \text{tabular data},
\end{cases}
\]
which gives the coordinate TCE reported as TCE in the main text:
\begin{equation}
    \mathrm{TCE}_{\mathrm{coord}}(p)
    = \frac{1}{d}\sum_{j=1}^d |L_{F_j}(p)|.
\end{equation}
To assess multivariate tail behavior beyond individual coordinates, we additionally define
\begin{align}
    F_{\mathrm{rad}}(x) &= \|x\|_2,\\
    F_{\mathrm{max}}(x) &=
    \begin{cases}
        \max_j |x_j|, & \text{synthetic data},\\
        \max_j x_j, & \text{tabular data},
    \end{cases}
\end{align}
and report
\begin{align}
    \mathrm{TCE}_{\mathrm{rad}}(p) &= |L_{F_{\mathrm{rad}}}(p)|,\\
    \mathrm{TCE}_{\mathrm{max}}(p) &= |L_{F_{\mathrm{max}}}(p)|,\\
    B_{\mathrm{max}}(p) &= L_{F_{\mathrm{max}}}(p).
\end{align}
The first two measure absolute radial and maximum-event tail discrepancies.
The signed quantity \(B_{\mathrm{max}}\) additionally distinguishes missing tail mass from excess tail mass.

\paragraph{Integrated TCE.}
For each TCE curve \(E\in\{\mathrm{TCE}_{\mathrm{coord}},
\mathrm{TCE}_{\mathrm{rad}},\mathrm{TCE}_{\mathrm{max}}\}\), we summarize the full tail scan by trapezoidal integration over the reference-quantile logit scale:
\begin{equation}
    \mathcal A_E
    = \sum_{i=1}^{20}
    \frac{E(p_i)+E(p_{i+1})}{2}
    \bigl(u_{i+1}-u_i\bigr),
    \qquad
    u_i=\operatorname{logit}(1-p_i).
\end{equation}
Lower TCE and TCE area indicate better tail coverage.

\paragraph{RBF-MMD\texorpdfstring{\(^2\)}{2}.}
To assess global distributional agreement, we use the biased RBF-MMD\(^{2}\) estimator.
Let \(X_m\) and \(Y_m\) contain \(m\) evaluation samples, with \(m=1{,}000\) for the synthetic and wildfire datasets and \(m=512\) for HRRR. We compute
\begin{align}
    \widehat{\mathrm{MMD}}_{\mathrm b}^{2}(X_m,Y_m)
    &= \frac{1}{m^2}\sum_{i,j=1}^{m}k(x_i,x_j)
     + \frac{1}{m^2}\sum_{i,j=1}^{m}k(y_i,y_j)
     - \frac{2}{m^2}\sum_{i,j=1}^{m}k(x_i,y_j),\\
    k(a,b) &= \exp\!\left(-\gamma\|a-b\|_2^2\right).
\end{align}
The estimator includes diagonal kernel terms. The bandwidth is fixed from the reference sample only,
\begin{equation}
    \gamma=\frac{1}{2m_2},
\end{equation}
where \(m_2\) is the median positive squared pairwise distance in \(X_m\).
The same bandwidth is then used for all methods within a dataset and trial.
Lower MMD\(^{2}\) indicates closer global distributional agreement.

\paragraph{Paired changes.}
RTR and its corresponding DDPM or DDIM baseline are evaluated on the same generated endpoints before and after correction.
For any lower-is-better metric \(\mathcal M\), the change barplots report
\begin{equation}
    \Delta\mathcal M
    = \mathcal M(X,Y_{\mathrm{RTR}})
    - \mathcal M(X,Y_{\mathrm{vanilla}}).
\end{equation}
Hence, negative changes indicate improvement.
The corresponding absolute metric values are reported separately in \autoref{sec:supp-exp-additional}.

\section{Further discussion of related work}
\label{appendix-related-work}

\paragraph{Heavy-tailed generative models beyond diffusion.}
The generation of heavy-tailed data has been tackled through different generative modeling paradigms, including adversarial models, normalizing flows, and variational autoencoders.
One strategy separates marginal modeling from dependence learning: Quant GANs \citep{wiese2020quant} and evtGAN \citep{boulaguiem2022modeling} transform the marginal distributions before adversarial training and subsequently map generated samples back to the original scale.
In particular, evtGAN combines extreme-value models for the marginals with a GAN for the dependence structure.
A second strategy introduces heavy-tailed latent distributions, as in \citet{feder2020nonlinear} and Pareto GAN \citep{huster2021pareto}.
This addresses a representational limitation of generators that are globally Lipschitz: such maps cannot transform Gaussian inputs into distributions with polynomial tails \citep{liang2022fat}.
Other methods incorporate extreme-value structure directly into the model parameterization.
EV-GAN learns a representation of extreme quantiles using ReLU networks \citep{allouche2022evgan}, while \citet{hasan2022modeling} develop neural representations of stable tail dependence functions.
Variational autoencoders have likewise been adapted to multivariate and spatial extremes \citep{lafon2023vae,zhang2024flexible}.

Related developments in normalizing flows address both marginal tail behavior and its variation across directions.
\citet{liang2022fat} introduce anisotropic tail-adaptive flows, whereas \citet{hickling2025flexible} retain a Gaussian base distribution and introduce a final transformation capable of producing heavy tails.
A complementary approach targets quantities relevant to a particular application: Tail-GAN uses a training objective based on Value-at-Risk and Expected Shortfall to reproduce portfolio tail risk \citep{cont2025tail}.
These contributions show how tail information can inform preprocessing, latent distributions, model architectures, and training objectives.
Our approach instead uses the data to estimate a radial tail correction for an already trained generator, preserving its angular component without updating its parameters.

\paragraph{Fine-tuning pretrained generative models.}
A substantial literature studies how pretrained generative models can be adapted to improve sample quality, satisfy additional constraints, or specialize their outputs to downstream tasks.
Approaches include supervised fine-tuning on reward-filtered data \citep{wu2023human}, reinforcement learning \citep{fan2023dpok,black2024training}, and optimization through the sampling procedure.
For example, \citet{marion2025implicit} develop a single-loop method for optimization problems involving distributions defined implicitly by stochastic samplers, with applications to diffusion fine-tuning.
Stochastic optimal control provides another formulation of this adaptation problem: Adjoint Matching learns a control through a regression objective \citep{domingo2025adjoint}, while Adjoint Sampling extends this framework to learning samplers for unnormalized densities \citep{havens2025adjoint}.
These methods reuse a pretrained model but involve an additional optimization procedure to adapt its generative dynamics.
Our objective is more specific: we seek to correct tail behavior through a transformation of generated samples.

\paragraph{Guidance and auxiliary correction models.}
Adaptation can also retain the pretrained generator and introduce an auxiliary correction during sampling.
Doob's $h$-transform provides a probabilistic foundation for this construction: for a positive terminal weighting function, the conditional expectation of that weight along a reference diffusion determines the ideal drift correction.
DEFT learns a compact approximation of the generalized $h$-transform \citep{denker2024deft}, while \citet{ouyang2026alignment} develop guidance based on conditional reward expectations for diffusion and flow-matching models.
Recent work further investigates regularized estimation of the Doob correction \citep{chang2026doob,suchkov2026weirdo} and its incorporation into model parameters through distillation \citep{guo2026dohf}.
These approaches illustrate several ways to reuse pretrained models, although they generally require access to the sampling dynamics and, in some cases, the training of an additional network.

\paragraph{Particle methods for inference-time adaptation.}
Sequential Monte Carlo provides an alternative means of adapting a pretrained diffusion model through weighted particles and resampling.
\citet{wu2023practical} develop an asymptotically exact particle approximation of a conditional distribution specified through a diffusion model, and \citet{cardoso2024monte} apply this approach to Bayesian linear inverse problems.
Feynman--Kac correctors extend weighted simulation to annealed distributions, guidance, and combinations of pretrained models \citep{skreta2025feynman}.
These methods address sampling from a specified modification of a reference distribution.
Their sampling guarantees do not, by themselves, establish that the reference model accurately represents the tails of the data distribution.
We refer to \citet{strasman2026Theoretical} for a theoretical analysis of the convergence of particle methods in the context of diffusion models.
This distinction motivates corrections that explicitly use the data to identify discrepancies in tail behavior.

\paragraph{Correction using the original data.}
A complementary perspective uses the original data to estimate features that a trained generator has failed to reproduce.
Discriminator rejection sampling \citep{azadi2019discriminator} and Metropolis--Hastings GANs \citep{turner2019metropolis} use discriminator-based density-ratio estimates to correct generated samples without changing the generator.
For diffusion models, \citet{kim2023discriminator} learn a data-versus-model discriminator after pretraining and use it to correct the score during sampling.
\citet{beltranvelez2026reward} instead use a discriminator-derived density-ratio reward to fine-tune a flow-matching model.
Our method follows this broader principle of using data to correct a trained generative distribution, but concentrates on its tail.
It applies a one-dimensional transformation to the radial component in a region of approximate radial--angular independence, preserving the angular component and leaving the pretrained generator unchanged.

\end{document}